\documentclass[]{article}
\usepackage{a4wide}
\usepackage{amsmath}
\usepackage{amssymb}
\usepackage{amsthm}
\usepackage{hyperref}
\usepackage{mathtools}
\usepackage[toc,page]{appendix}
\usepackage[ruled,vlined,linesnumbered]{algorithm2e}
\usepackage{graphicx}
\usepackage{authblk}
\usepackage{xcolor}
\usepackage{bm}
\usepackage{subfigure}
\usepackage[final]{showlabels}
\usepackage[T1]{fontenc}

\usepackage{natbib}

\newtheorem{example}{Example}
\newtheorem{theorem}{Theorem}
\newtheorem{lemma}{Lemma}
\newtheorem{definition}{Definition}
\newtheorem{claim}{Claim}
\newtheorem{proposition}{Proposition}
\newtheorem{corollary}{Corollary}

 \theoremstyle{definition}
\newtheorem{remark}{Remark}
\usepackage{enumitem}

\newcommand*{\ditto}{---\textquotedbl---} % available in T1 encoding

\newcommand{\Oh}{\mathcal{O}}
\newcommand{\X}{\mathcal{X}}
\newcommand{\N}{\mathbb{N}}

\newcommand{\LL}{\mathcal{L}}
\newcommand{\K}{\mathcal{K}}
\newcommand{\ind}{\boldsymbol{1}}

\DeclarePairedDelimiter\floor{\lfloor}{\rfloor}

\DeclareMathOperator*{\argmax}{arg\,max}
\DeclareMathOperator*{\argmin}{arg\,min}

\DeclareMathOperator*{\argsup}{arg\,sup}

\title{Gaussian Process Bandits with Adaptive Discretization}
\date{}
\author{Shubhanshu Shekhar\thanks{shshekha@eng.ucsd.edu}   \hspace{0.1em} and Tara Javidi\thanks{tjavidi@eng.ucsd.edu}}
\affil{Department of Electrical and Computer Engineering,\\ University of California, San Diego}

\begin{document}
\maketitle

\begin{abstract}
In this paper, the problem of maximizing a black-box function $f:\X \to \mathbb{R}$ is studied in the Bayesian framework with a Gaussian Process (GP) prior. In particular, a new algorithm for this problem is proposed, and high probability bounds on its simple and cumulative regret are established. The query point selection rule in most existing methods involves an exhaustive search over an increasingly fine sequence of uniform discretizations of $\X$. The proposed algorithm, in contrast, adaptively refines $\X$ which leads to a lower computational complexity, particularly when  $\X$ is a subset of a high dimensional Euclidean space. In addition to the computational gains, sufficient conditions are identified under which the regret bounds of the new algorithm improve upon the known results. 
Finally an extension of the algorithm to the case of contextual bandits is proposed, and high probability bounds on the contextual regret are presented. 

\end{abstract}

%-----------------
\section{Introduction}
\label{section:introduction}
We consider the problem of maximizing a function $f:\X \to \mathbb{R}$ from its noisy observations of the form 
\begin{equation}
\label{eq:observation_model}
y_t = f(x_t) + \eta_t,  \ \ \ t=1,2,\ldots,n,
\end{equation}
where $\eta_t$ is the observation noise at time $t$. 
We work in the Bayesian setting, assuming  that the function $f$ is a sample from a zero mean Gaussian Process (GP) indexed by the  space $\X$, and $\eta_t$ for $t\geq 1$ are i.i.d. $N(0,\sigma^2)$ Gaussian random variables. We further assume that the function $f$ is expensive to evaluate, and we are allocated a budget of $n$ function evaluations.

% Assuming that the function $f$ is a sample from a zero mean Gaussian Process $GP(0,K)$ \tj{redundant? delete?}{indexed by the search space $\X$} and $\eta_t,  t=1,\ldots,n,$ are iid $N(0,\sigma^2)$ Gaussian random variables, 
% the goal is to design a query point selection strategy which reliably learn about a true maximizer $x^*$ of $f$. 
This problem can be thought of as an extension of the Multi-armed bandit (MAB) problem to the case of infinite (possibly uncountable) arms indexed by the set $\X$, and is referred to as the GP bandits problem \citep{srinivas2012information}.
The goal is to design a strategy of sequentially selecting query points $x_t \in \X$ based on the past observations $\{(x_i,y_i) ; 1\leq i \leq t-1\}$ and the prior on $f$. 
As in the case of MAB with finitely many arms,  the performance of any query point selection strategy is usually measured by the cumulative regret $\mathcal{R}_n$, which forces the agent to address the \emph{exploration-exploitation}
trade-off:
\begin{equation}
\label{eq:cumulative_regret}
\mathcal{R}_n = \sum_{t=1}^nf(x^*) - f(x_t).
\end{equation}
An alternative measure of performance is the \emph{simple regret} $\mathcal{S}_n$ which is used in the Bayesian Optimization (BO) or the pure exploration problem:
\begin{equation}
 \label{eq:simple_regret}
 \mathcal{S}_n = f(x^*) - f(x_n).
\end{equation}
% Here $x(n)$ is the point recommended by the algorithm at the end of $n$ rounds. 

%Clearly, we must impose some regularity assumptions on the unknown function $f$ for this problem to be tractable. In this paper, we take the Bayesian approach of assuming that the function $f$ is a sample from a zero mean Gaussian Process $GP(0,K)$ indexed by the search space $\X$. Appropriate restrictions on the covariance function $K$ then result in  the required assumptions on $f$ in a probabilistic manner. 

\subsection{Prior work}
\label{subsec:prior_work}

Optimizing a black-box function from its noisy observations is an active area of research with a large body of literature. Here, we review  existing methods which take a Bayesian approach with GP prior to this problem, and have provable guarantees on their performance. 

\cite{srinivas2012information} formulated the task of black-box function optimization as a MAB problem and proposed the GP-UCB algorithm which is a modification of the Upper Confidence Bound (UCB) strategy widely used in bandit literature.  The algorithm constructs high probability UCBs on the function values using the GP posterior and selects the evaluation points by maximizing the UCB over $\X$. 
For finite search spaces $\X$ they showed that the GP-UCB algorithm admits a high probability upper bound on the cumulative regret of the form:\begin{equation}
\label{eq:information_type}
\mathcal{R}_n \leq \mathcal{O}(\sqrt{n\log(n)\gamma_n}),
\end{equation} where $\gamma_n$ is the \emph{maximum information gain} with $n$ evaluations. We will refer to cumulative regret bounds of this form as \emph{information-type} regret bounds in this paper. In addition, to make the dependence on $n$ explicit, \cite{srinivas2012information} further derived bounds on the term $\gamma_n$ for some commonly used kernels. Finally, they presented an extension of the GP-UCB algorithm to the case of continuous $\X$ by applying it on a sequence of increasingly fine uniform discretizations of $\X$.

Follow up works to \cite{srinivas2012information} have extended the GP-UCB algorithm in several ways. 
\cite{contal2016stochastic} proposed a method of constructing a sequence of uniform discretizations with tight control over the approximation error, which allowed the extension of the GP-UCB algorithm to arbitrary compact metric spaces $\X$. 
\cite{desautels2014parallelizing} and \cite{contal2013parallel} considered  the GP bandits problem with the additional assumption that the evaluations can be performed in parallel. 
\cite{desautels2014parallelizing} proposed the GP-BUCB algorithm which selects the points in a batch sequentially by maximizing a variant of the UCB, which is computed by keeping the mean function fixed and only updating the posterior variance. \cite{contal2013parallel} proposed the GP-UCB-PE which uses the UCB function for selecting the first point of a batch, and then proceeds in a greedy manner selecting the remaining points by maximizing the posterior variance. 
\cite{krause2011contextual} proposed and analyzed the CGP-UCB algorithm for the contextual GP bandits problem, where the mean reward function corresponding to context-action pairs is modeled as a sample from a GP on the context-action product space.   
\cite{kandasamy2016multi} considered a multi-fidelity version of the GP bandits problem in which they assumed the availability of a sequence of approximations of the true function $f$ with increasing accuracies which were cheaper to evaluate. They proposed an extension of GP-UCB called the MF-GP-UCB and derived information-type bounds on its cumulative regret.

\cite{wang2016estimation} proposed the GP-EST algorithm which looks at the optimization problem through the lens of estimation. In particular, the algorithm constructs an estimate of the maximum function value $f(x^*)$, and then selects a point for evaluation which has the largest probability of  attaining this value. 
\cite{russo2014thompson} analyzed the performance of the Thompson Sampling algorithm to a large class of problems, including the GP bandits problem. 
Thompson Sampling is a randomized strategy in which query points are sampled according to the posterior distribution on $x^*$. Since computing the posterior on $x^*$ may be complicated, in practice, the query points are selected in the following two step procedure: first, a sample $\tilde{f}_t$ of the unknown function $f$ is generated, and then the query point $x_t$ is chosen by maximizing $\tilde{f}_t$ over $\X$. For the case of continuous $\X$, the function samples are generated over uniform discretizations $\X_t$ of $\X$. By observing a relation between the expected regret of Thompson Sampling and UCB strategies, \cite{russo2014thompson} obtained information-type bounds on the expected cumulative regret of the Thompson Sampling algorithm for GP bandits. 

As observed in \citep{bubeck2011pure}, bounding the cumulative regret automatically gives us a bound on the expected simple regret by employing a randomized point recommendation strategy.
Additionally, for the pure exploration setting, several algorithms specifically geared towards minimizing $\mathcal{S}_n$, such as \emph{Expected Improvement (GP-EI)}, \emph{Probability of Improvement(GP-PI)}, \emph{Entropy Search} and \emph{Bayesian Multi-Scale Optimistic Optimization
(BaMSOO)}   have been proposed (see \citep{shahriari2016survey} for a recent survey). \cite{bogunovic2016truncated} considered the  BO and Level Set Estimation problems in a unified manner and proposed the \emph{Truncated Variance Reduction (TRUVAR)} algorithm which selects evaluation points greedily to obtain the largest reduction in the sum of truncated variances of the potential maximizers. The performance of all these algorithms have been empirically studied over various synthetic as well as real-world datasets. Furthermore,  theoretical guarantees are also known for GP-EI \citep{bull2011convergence} and BaMSOO\citep{wang2014bamsoo} with noiseless observations, and for TRUVAR \citep{bogunovic2016truncated} with noisy observations and non-uniform cost of evaluations. 

% The performance of these algorithms have been empirically studied under noisy as well as noiseless observations. However, theoretical guarantees are only known for GP-EI~\citep{bull2011convergence} and BaMSOO~\citep{wang2014bamsoo} in the restricted setting of noiseless observations. 

All the algorithms above, with the exception of BaMSOO, require solving an \emph{auxiliary optimization} problem in each round $t$ for selecting the query point $x_t$. The objective function of this auxiliary optimization problem is usually non-convex and multi-modal and hence requires  an exhaustive search over an increasingly fine sequence of uniform discretizations  to guarantee that a close approximation of the true optimum is found \citep{srinivas2012information,contal2016stochastic}.
The size of these uniform discretizations increases exponentially with the dimension of $\X$. This is because these discretizations are chosen off-line and do not depend on the function evaluations made up to round $t$.
% In practice, the exhaustive search may be infeasible, and several approximation methods are employed, which however lead to the violation of conditions required for the theoretical results. 
In contrast, BaMSOO adaptively constructs  discretizations  by locally refining the regions of $\X$ in which $f$ is more likely to take higher values based on the observations. As a result, the  size of the discretizations under BaMSOO are independent of the dimension of $\X$ which leads to significantly lower computational costs when $\X$ is  high dimensional. Our work is strongly motivated by this aspect of BaMSOO to provide the first algorithm for GP bandits  with noisy observations whose computational complexity remains independent of the dimension of $\X$. 
% All the algorithms mentioned above, with the exception of BaMSOO, select their query point $x_t$ by solving an \emph{auxiliary optimization} problem in every round $t$.
% Solving this optimization problem for high dimensional $\X$ can be computationally infeasible. In practice, approximation techniques are employed~\cite{shahriari2016survey}, which lead to violation of conditions required for the theoretical results. 

\subsection{Our contributions}
\label{subsec:our_contributions}
In this paper, we address two issues with existing approaches to the GP bandits problem:
\begin{enumerate}
\item As discussed above, all the existing algorithms for GP bandits require solving an auxiliary optimization problem over the entire search space 
for selecting a query point which may be computationally infeasible, and thus practical implementations resort to various approximation techniques which do not come with theoretical guarantees. 

\item Furthermore, by constructing specific Gaussian Processes we show that the information-type regret bounds can be too pessimistic, thus motivating the need for designing algorithms that admit alternative analysis techniques. 

\end{enumerate}
 To tackle these two problems, we design algorithms for GP bandits which utilize ideas from existing works  in the Lipschitz function optimization literature, such as \citep{bubeck2011x,munos2011doo,munos2014bandits, kleinberg2013bandits}. 
More specifically, our main contributions are as follows:
\begin{itemize}

\item We first present an algorithm for GP bandits which employs a tree of partitions of the search space $\X$ to adaptively refine it based on observations. We show that because of the adaptive discretization, when $\X \subset 
\mathbb{R}^D$ and $D$ is large, our algorithm has significantly less computational complexity than algorithms requiring auxiliary optimization. 

\item We obtain high probability bounds on the cumulative regret of our algorithm which are always as good as, and in some cases strictly better than, the existing regret bounds. In particular, we obtain the first explicit sublinear
regret bounds for the GP with exponential kernel (Ornstein-Uhlenbeck process) and also identify sufficient conditions under which our bounds improve upon the current ones for Mat\'ern family of kernels.

\item We also derive high probability bounds on the simple regret for our algorithm. To the best of our knowledge, BaMSOO~\citep{wang2014bamsoo} is the only  adaptive\footnote{we use the term \emph{adaptive} to
refer to algorithms which adaptively discretize the search space $\X$ based on earlier observations } algorithm for the black-box optimization problem in the Bayesian setting, for which 
theoretical guarantees on simple regret are known. Our algorithm matches BaMSOO's performance  with the additional advantages that it requires fewer assumptions on the covariance functions and can work with noisy observations. 
% Furthermore, when applied to the noiseless setting our algorithm improves upon BaMSOO and exhibits an exponentially decaying $\mathcal{S}_n$ for a   class of Gaussian Processes that are described in 
% Appendix~\ref{appendix:sufficient_condition}. 

\item We also study two extensions of our algorithm. First, we present a Bayesian Zooming algorithm based on \citep{kleinberg2013bandits,slivkins2014contextual} and obtain theoretical guarantees on its regret performance. 
This algorithm assumes a \emph{covering oracle} access to the metric space $\X$ instead of requiring a hierarchical tree of partitions of $\X$. We then extend our algorithm for GP bandits to the contextual GP bandits and 
obtain bounds on the contextual regret. 

\item Finally, our algorithms and the theoretical bounds rely on a set of technical results about Gaussian Process which may be of independent interest. We provide these results and discuss their implications in Section~\ref{section:technical_results}.

\end{itemize}

\subsection{Toy examples}
\label{subsec:toy_examples}

As mentioned earlier, our cumulative regret bounds for Mat\'ern kernels improve upon the known information type bounds for GP bandits. In this section, we attempt to provide some intuition for this result. 
In particular, we construct two toy examples which serve to highlight a potential drawback of the information type regret bounds for GP bandit problems shown in (\ref{eq:information_type}).

The information-type regret bounds (\ref{eq:information_type}) depend on the \emph{maximum information gain} $\gamma_n$ which is defined as:
\begin{equation}
\label{eq:gamma_n}
\gamma_n = \sup_{x[1:n] \in  \X^n}I(f;y_{x[1:n]}), 
\end{equation}
Here $I(f;y_{x[1:n]})$ is the mutual information between the unknown function $f$ and vector of observations $y_{x[1:n]}$ corresponding to the $n$ query points $x[1:n]$. 
This term depends on the covariance function\footnote{we will use the terms \emph{covariance functions} and
\emph{kernels} interchangeably} of the Gaussian Process (GP),  and upper bounds on $\gamma_n$ for many commonly used GPs are given in \citep{srinivas2012information}. 
We note that since our aim is to gather information about a maximizer $x^*$ of $f$, and not necessarily 
about the behavior of $f$ over the entire space $\X$, information-type regret bounds can  be quite loose. We present two examples which have been specifically constructed to illustrate the scenarios where the regret bounds implied by (\ref{eq:information_type}) are very pessimistic. Both  examples utilize the fact that the maximum information gain ($\gamma_n$) can be large if the Gaussian Process has many independent components, even when the maximizer may be simple to learn. 

For our first example, we  construct a GP whose samples have simple structure around the maximum despite the highly complex structure away from the maximizer. More specifically, we begin by  dividing the interval $[0,1]$ into three equal subintervals. Over the second and third intervals, the GP sample varies smoothly as scaled and shifted versions of a smooth function $\varphi(\cdot)$, modulated by a Standard Normal random variable $X_1$. The first subinterval is further divided into three parts, and this process continues infinitely.

\begin{example}
 \label{ex:toy_example1}

Suppose $\X=[0,1]$ and let us  define a GP = $\{f(x)| x\in \X \}$ as follows:
\begin{equation}
\label{eq:toy_example1}
f(x) = \sum_{i=1}^{\infty}a_iX_i\bigg(\varphi\bigg(\frac{x}{b_i}-1\bigg) -\varphi\bigg(\frac{x}{b_i}-2\bigg)\bigg),
\end{equation}
where $(a_i)_{i\geq 1}$ is a non-increasing positive sequence, $b_i = 3^{-i}$ for $i\geq 1$, $\varphi:[0,1]\to [0,1]$ is a continuous unimodal function with $\varphi(0)=\varphi(1)=0$ and $\varphi(0.5)=1$, and $(X_i)_{i=1}^{\infty}$ are a 
sequence of independent Standard Normal random variables.

For this GP, we can claim the following (details in Appendix-\ref{appendix:toy_example1}):
\begin{itemize}
\item For the choice of $a_i$ described in Appendix~\ref{appendix:toy_example1}, we have $\gamma_n = \Omega\big(\frac{n\sigma^2}{\log(n)}\big)$, which means that the information-type bound (\ref{eq:information_type}) on the cumulative regret is linear in $n$. 

\item On the other hand, if $a_1 >> a_i$ for $i \geq 2$, then the true maximizer $x^* \in \{ 1/2,5/6 \}$ with high probability, and it can be identified with just one function evaluation implying a constant cumulative regret, $\mathcal{R}_n \leq \Oh(1)$. 
\end{itemize}

\end{example}

For our second example, we construct a GP in which   the search space is partitioned at different scales,  and statistically equivalent components are assigned to the sets of a given partition. This process  is repeated with 
increasingly finer partitions, and we show that for certain choice of parameters, each observation of the GP sample results in diminishing  the region of uncertainty associated with $x^*$ by a constant factor. However, the 
information-type bound again is dominated by the information obtained from the large number of independent components of the GP and gives a linear  upper bound on the cumulative regret.

\begin{example}
 \label{ex:toy_example2}

We again take $\X = [0,1]$ and let $\varphi_1$ denote the following function 
\[
\varphi_1(x) = 
\begin{cases}
\varphi(3x) & \text{if $x \in [0,1/3)$}\\
\varphi(3x-1) & \text{if $x \in [1/3,2/3)$}\\
-\varphi(3x-2) & \text{if $x \in [2/3,1]$  }
\end{cases}
\]
where $\varphi$ is the function  used in Example~\ref{ex:toy_example1}.
Let us now define a $GP = \{ f(x) | x \in \X \}$ recursively as follows:
\begin{equation}
\begin{aligned}
f_i(x) &= a_iX_i\varphi_1(x) + f_{i+1}(3x) + f_{i+1}(3(x-2/3)) \hspace{1em} \text{for } i\geq 2\\
f(x) &= a_1X_1 \varphi_1(x) + f_2(3x) + f_2(3(x-2/3)). 
\end{aligned}
\end{equation}
As before $(a_i)_{i\geq 1}$ is a decreasing sequence of positive real numbers, and $(X_i)_{i\geq 1}$ are i.i.d. Standard Normal random variables. For this example, we can claim the following:

\begin{itemize}
\item If the   noise variance $\sigma^2$ is small enough,  we have $\gamma_n \geq \Omega\big(n \big)$ which implies a linear in $n$ information-type bound on cumulative regret. 

\item With the choice of parameters $(a_i)_{i\geq 1}$ described in Appendix~\ref{appendix:toy_example2}, we can select the evaluation points in such a way  that with high probability after every observation, the size of the region containing $x^*$ shrinks by a factor of $3$, which in turn implies that the cumulative regret satisfies $\mathcal{R}_n \leq \mathcal{O}(\log n)$. 
\end{itemize}
\end{example}

Both our examples have been specifically crafted to highlight scenarios in which the information type upper bounds given in (\ref{eq:information_type}) may not reflect the actual performance of the algorithms due to its 
dependence on the term $\gamma_n$. In Section~\ref{subsec:improved_bounds_for_matern} we further strengthen this observation by showing that the information-type regret bounds are loose for a practically relevant class of Gaussian Processes. 

\begin{figure}%
\centering
\subfigure[Example~\ref{ex:toy_example1}]{%
\label{fig:first}%
\includegraphics[height=2in]{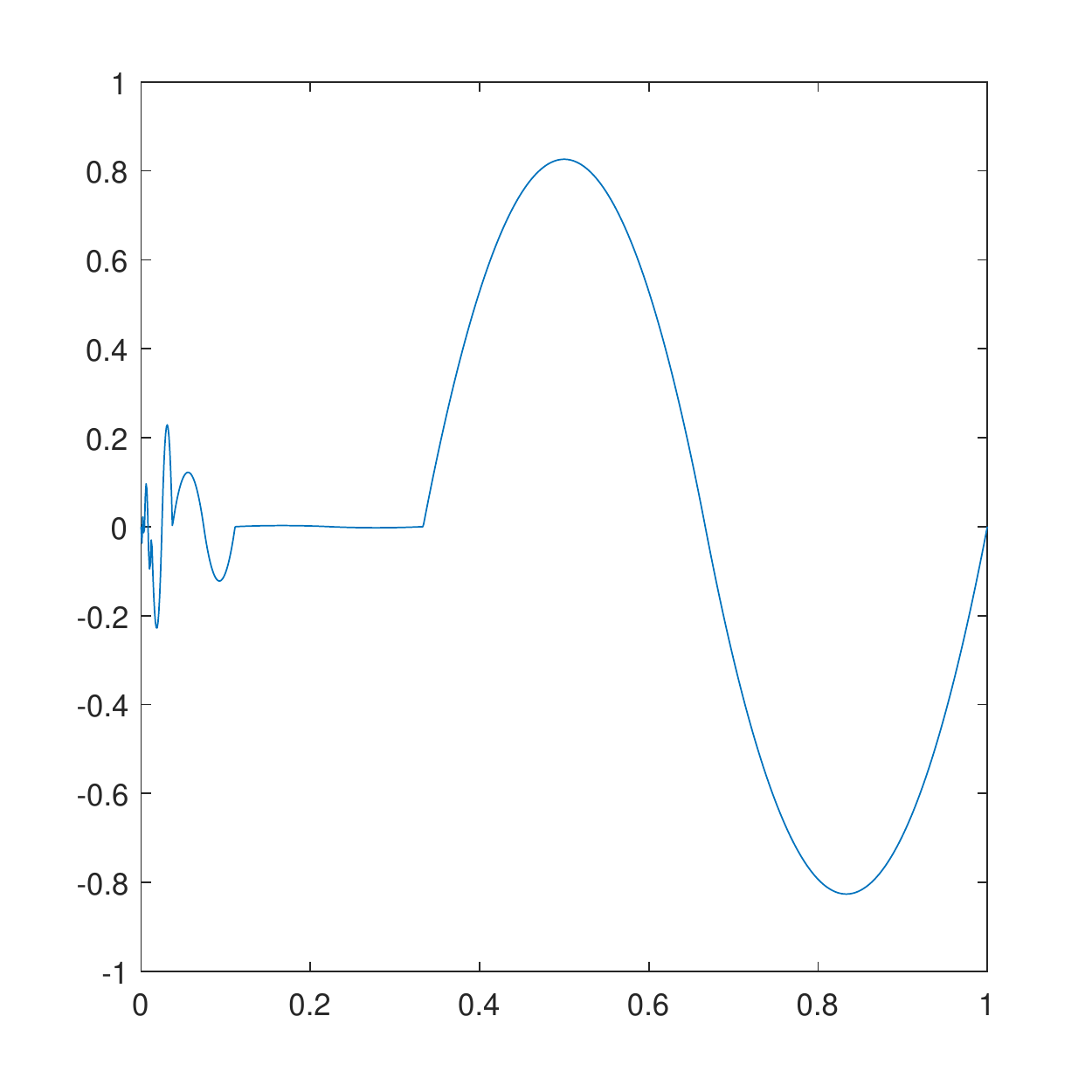}}%
\qquad
\subfigure[Example~\ref{ex:toy_example2}]{%
\label{fig:second}%
\includegraphics[height=2in]{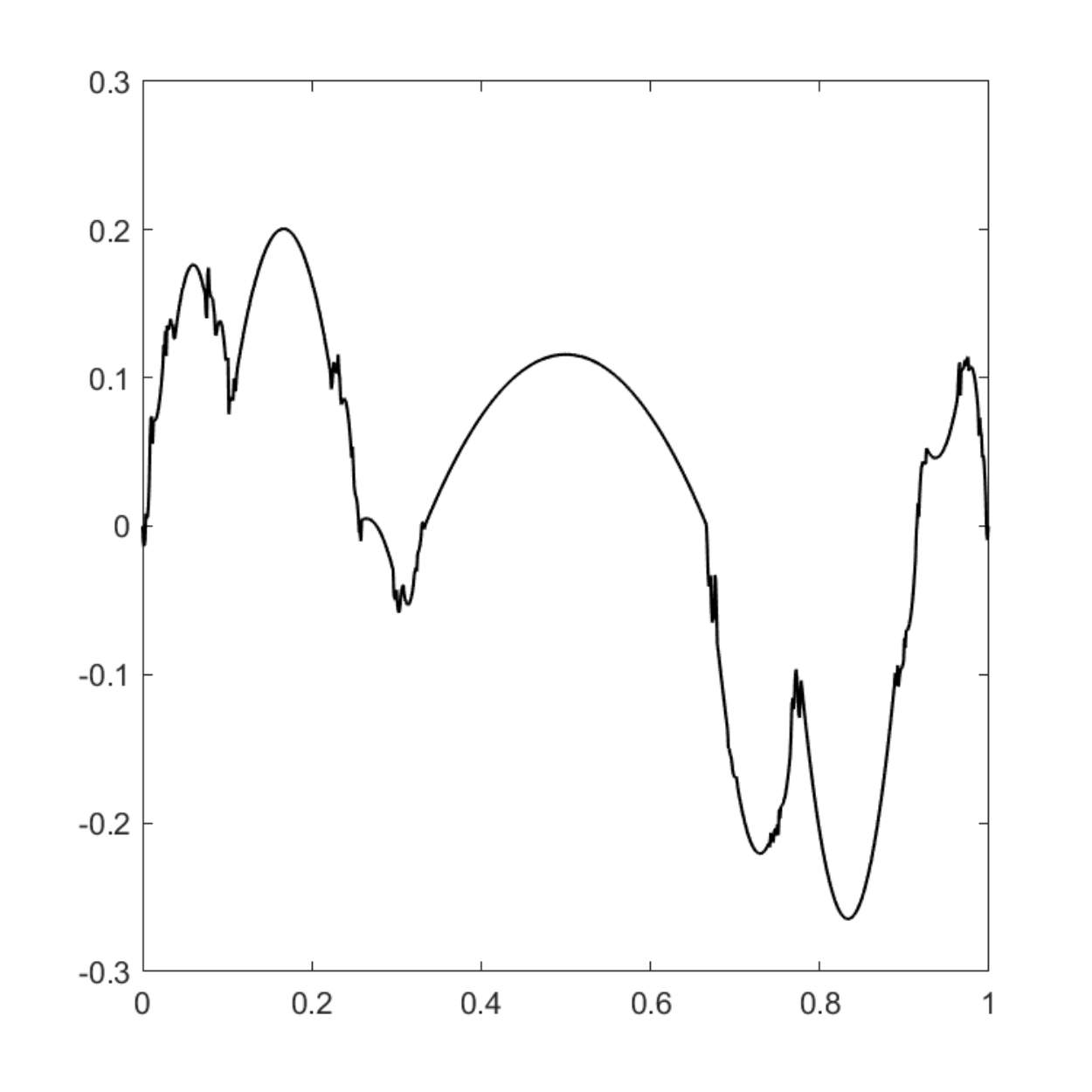}}%
\caption{ An instance of the sample paths of the two toy Gaussian Processes. }
\end{figure}

The rest of the paper is organized as follows: In Section~\ref{section:preliminaries} we introduce the required definitions and present some background for the problem. We then describe our algorithm for GP bandits and analyze its 
regret in Section~\ref{section:main_algorithm}. We discuss the behavior of our algorithm in some specific problem instances in Section~\ref{section:discussion}. In Section~\ref{section:extensions} we study two extensions of our 
approach and analyze their performance. Finally, Section~\ref{section:technical_results} contains some technical results which were used in designing our algorithms.

\section{Preliminaries}
\label{section:preliminaries}

In this section we recall some definitions required for stating the results, and fix the notations used. 
	
\begin{definition}A Gaussian Process is a collection of random variables $\{ f(x); x \in \X\}$ which satisfy the property that $(f(x_1),f(x_2),\ldots,f(x_m))$ is a jointly Gaussian random variable for all $\{x_1,x_2,\ldots,x_m\} \subset \X$ and $m \in \mathbb{N}$. 
A Gaussian Process is completely specified by its mean function $\mu(x) = \mathbb{E}[f(x)]$ and its covariance function $K(x_1,x_2) = \mathbb{E}[(f(x_1)-\mu(x_1))(f(x_2)-\mu(x_2))]$. 
\end{definition}
For a comprehensive discussion about Gaussian Processes and their applications in machine learning, see \citep{rasmussen2006gaussian}.

\begin{remark}
	Any zero mean Gaussian Process with covariance function $K$ induces a metric $d$ on its index set $\X$, defined as 
	\begin{align}
    \label{eq:gp_metric}
	d(x_1,x_2) &= \mathbb{E}[(f(x_1)-f(x_2))^2]^{1/2} \\ &= \big( K(x_1,x_1) + K(x_2,x_2) - 2K(x_1,x_2) \big)^{1/2}.
	\end{align}
	which gives us the following useful tail bound for any $x_1,x_2 \in \X$ and $a \geq 0$:
	\begin{equation}
	\label{eq:gaussian_tail}
	Pr( |f(x_1)-f(x_2)|\geq a ) \leq 2\exp\bigg(-\frac{a^2}{2d(x_1,x_2)}\bigg).
	\end{equation}
\end{remark}

Next, we introduce some properties of  any metric space $(X,l)$ which will be used later on.
	\begin{definition}
		\label{def:NM}
		Suppose $\X$ is a non-empty set and $l$ is a metric on $\X$. Then we have the following:
		\begin{itemize}
			\item A subset $\X_1$ of $\X$ is called an \underline{$r$-covering} set of $\X$ if for any $x \in \X$, we have $l(x,\X_1) \leq r$ where $l(x,\X_1) \coloneqq \inf \{l(x,y): y \in \X_1\}$. The cardinality of the smallest such $\X_1$ is called the $r$-covering number of $\X$ with respect to $l$, denoted by $N(\X,r,l)$. 
			\item The \underline{metric dimension} of a space $\X$ with associated metric $l$ is the smallest number $D_1$ such that we have for all $r>0$, 
		\[ N(\X,r,l) \leq Cr^{-D_1}\] for some $C>0$.
        
		\end{itemize}

	\end{definition}

For bounded subsets of $\mathbb{R}^D$ with a metric $l$, the metric dimension coincides with the usual notion of dimension \citep[page 125]{van2014probability}.
The metric dimension $D_1$  gives us a notion of dimensionality intrinsic to the metric space $(\X,l)$. We now present a function specific measure of dimensionality of $(\X,l)$.

\begin{definition}
\label{def:near_opt}Suppose $\X$ is a non-empty set, $l$ is a metric on $\X$ and $f$ is a function from $\X$ to $\mathbb{R}$. Then
	\begin{itemize}
		\item A subset $\X_2$ of $\X$ is called an $r$-separated set of $\X$ if for any $x_1,x_2 \in \X_2$ we have $l(x_1,x_2)\geq r$. The cardinality of the largest such set $\X_2$ is called the $r$-packing number of $\X$ with respect to $l$, and is denoted by $M(\X,r,l)$. 
		\item For any $r>0$ and $\zeta:\mathbb{R}^+\to \mathbb{R}^+$, consider the $\zeta(r)$ near-optimal set $\X_{\zeta(r)}\coloneqq \{ x \in \X: f(x) \geq f(x^*) - \zeta(r)  \}$ and its $r$-packing number $M(\X_{\zeta(r)},r,l)$. Then we define the ($\Delta_0,\zeta$)-\underline{near-optimality dimension} ($D^f(\Delta_0,\zeta)$) associated with $(\X,l)$ and the function $f$ as the smallest real number $m \geq 0$ such that for all $r \leq \Delta_0$, we have
        \begin{equation}
		\label{eq:near_opt}
		 M(\X_{\zeta(r)},r,l) \leq \tilde{C}r^{-m}, 
		\end{equation}
		
		for some $\tilde{C} >0$.
        
%         We define the \emph{ ($\Delta_0,\zeta$)-near-optimality dimension ($D^f(\Delta_0,\zeta)$)} associated with $(\X,l)$ and the function $f$ as the smallest real number $m \geq 0$ such that for all $r \leq \Delta_0$, the $r$-packing number of the set $\X_{\zeta(r)}= \{ x \in \X: f(x) \geq f(x^*) - \zeta(r)  \}$ for some $\Delta_0>0$ and $\zeta:\mathbb{R}^+\to \mathbb{R}^+$, is upper bounded as:
% 		\begin{equation}
% 		\label{eq:near_opt}
% 		 M(\X_{\zeta(r)},r,l) \leq \tilde{C}r^{-m} 
% 		\end{equation}
		
% 		for some $\tilde{C} >0$. 
	\end{itemize}

	\end{definition}

Our definition of the near-optimality dimension is based on similar definitions used in existing works in  literature such as \citep{bubeck2011x,munos2011doo,valko2013stochastic}. 

 \begin{remark}
We note that for any $(\X,l)$ with finite metric dimension $D_1$, by using volume arguments~\citep[Lemma-5.13]{van2014probability} we can show that $D^f(\Delta_0,\zeta) \leq D_1$. An example~\citep[Example~3]{bubeck2011x} where this inequality is strict is the following: consider $\X = [-1,1]$, $l(x_1,x_2)=|x_1-x_2|$ and $f(x) = 1- |x|^a$, and $\zeta(r) = r^b$ for some $0<b\leq a$. Then $D^f(1,\zeta) =  1-b/a \leq 1=D$, and in particular $D^f(1,\zeta) = 0$ for $a=b$.
\end{remark}

\begin{definition}
\label{def:well_behaved}
We will call a compact metric space $(\X,l)$  \underline{well-behaved} if there exists a sequence of subsets $(\X_h)_{h\geq 0}$ of $\X$ satisfying the following properties:

\begin{itemize}
\item [\textbf{P1}]Each subset $\X_h$ has $N^h$ elements for some $N>1$, i.e. $\X_h = \{x_{h,i}, 1\leq i \leq N^h\}$, and to each element $x_{h,i} \in \X_h$ is associated a \emph{cell} $X_{h,i} = \{x \in \X: l(x,x_{h,i}) \leq l(x,x_{h,j}) \text{ for all } j \neq i \}$.

\item [\textbf{P2}]For all $h\geq 0$ and $0\leq i\leq N^h$, we have 
\begin{equation}
\X_{h,i} = \cup_{j=N(i-1)+1}^{Ni} \X_{h+1,j}.
\end{equation}
The nodes $x_{h+1,j}$ for $N(i-1)+1 \leq j \leq Ni$ are called the \emph{children} of $x_{h,i}$, which in turn is referred to as the \emph{parent}. 

\item [\textbf{P3}] We assume that the cells have geometrically decaying radii, i.e., there exists $0 < \rho <1$ and $0<v_2\leq 1\leq v_1$ such that we have 
\begin{equation}
\label{ea:geometrically_decaying_raii}
B(x_{h,i},v_2\rho^h,l) \subset \X_{h,i} \subset B(x_{h,i},v_1\rho^h,l).
\end{equation}

\end{itemize}

\end{definition}

From $P1$ we can see that the cells $\{\X_{h,i}; 1 \leq i \leq N^h \}$ partition the space $\X$ for every $h \geq 0$, while $P2$ implies that we get an increasingly fine sequence of partitions with increasing $h$. Finally $P3$ imposes the condition that for any $h$, the points $x_{h,i}$ are evenly spread out in the space $\X$. The subsets $(\X_h)_{h\geq 0}$ satisfying these properties are said to form a tree of partitions~\citep{munos2014bandits,bubeck2011x}. 

\begin{remark}
We note that if $\X = [a,b]^D \subset \mathbb{R}^D$, and $l$ is any metric on $\X$, then $\X$ is well-behaved according to the above definition. The cells $\X_{h,i}$ in this case are $D$ dimensional hyper-rectangles such that  $\X_{h+1,j}$ for $1\leq j \leq N^h$ can be constructed from $\X_{h,i}$ by dividing it along its longest edge into $N$ equal parts. 
\end{remark}

\begin{table}[htbp]
\centering 
\begin{tabular}{ |c|l|l| } 
 \hline 
 Symbol	& Description	& Introduced in \\ \hline
 $f$	&	black-box function 	&	Section~\ref{section:introduction} \\
 $n$	& function evaluation budget & \ditto \\
  $\eta_t, \sigma^2$ & observation noise distributed as $N(0,\sigma^2)$ & \ditto \\
 $GP(0,K)$	& prior on $f$ with covariance function $K$	&  \ditto \\
 $\K$	& class of covariance functions considered & Section~\ref{subsubsec:assumptions_on_K}\\
 $\delta_K, \alpha,C_K,g$ & parameters associated with $K \in \K$ & \ditto \\ 
 & & \\
 $\mu_{t-1}, \sigma^2_{t-1}$ & posterior mean and variance functions & \\
 
 $\mathcal{R}_n$	& Cumulative regret & Section~\ref{section:introduction},(\ref{eq:cumulative_regret}) \\
 $\mathcal{S}_n$	& Simple regret &	Section~\ref{section:introduction},(\ref{eq:simple_regret}) \\
 $\gamma_n$			& Maximum Information gain & Section~\ref{subsec:toy_examples}, (\ref{eq:gamma_n})\\
 $\mathcal{R}_n^c$	& Contextual regret& Section~\ref{subsec:contextual_gp_bandits},	(\ref{eq:contextual_regret}) \\
 $\Delta(x)$ & $f(x^*) - f(x) $ & \\
 $\Delta^c(x_{\tau})$ & $\sup_{x^a \in \X_a}f(x_{\tau}^c,x^a) - f(x_{\tau}^c,x_{\tau}^a)$ &  Section~\ref{subsec:contextual_gp_bandits} \\

 & & \\
 $(\X,l)$ & Compact search space $\X$ with metric $l$	& \\
 $B(x,r,l)$ & $l$-ball with center $x$, and radius $r$ &  (\ref{ea:geometrically_decaying_raii}) \\
 $d$	& metric induced by $GP(0,K)$ on $\X$	& Section~\ref{section:preliminaries}, (\ref{eq:gp_metric}) \\
 $D_1$	& Metric dimension 	& \ditto,Definition~\ref{def:NM}  \\
 $N(\X,r,l)$ & Covering number & \ditto, \ditto\\
 $M(\X,r,l)$ & Packing number & \ditto, Definition~\ref{def:near_opt}\\
 $D^f(\Delta_0,\zeta)$ & $(\Delta_0,\zeta)$-near-optimality dimension & \ditto, \ditto\\
 $\tilde{D}, \tilde{D}_Z$ & instances of $D^f(\Delta_0,\zeta)$ & Remark~\ref{remark:near_optimality_dimension}, Remark~\ref{remark:zoom_near_opt}\\
 $N,v_1,v_2,\rho$	& Parameters of the tree of partitions & Section~\ref{section:preliminaries}, Definition~\ref{def:well_behaved} \\ \hline

 \multicolumn{3}{|l|}{Parameters of Algorithm~\ref{alg:gp_tree} and Algorithm~\ref{alg:contextual_tree}} \\ \hline

 $\LL_t$	& the set of leaf nodes & Section~\ref{subsec:tree_based_algo}\\
 $I_t(x_{h,i})$	& Index used for point selection in Algorithm~\ref{alg:gp_tree} & \ditto, (\ref{eq:ucb}) \\
 $\bar{U}_t(x_{h,i})$ & upper bound on $f(x_{h,i})$ & \ditto, (\ref{eq:function_upper_bound}) \\
 $p(x_{h,i})$ & parent node of $x_{h,i}$& \\
 $\beta_n$	& multiplicative factor for confidence intervals & Section~\ref{subsubsec:details_of_algo1}, Claim~\ref{claim:doo_Bn} \\
 $h_{\max}$ & maximum depth of the tree & \ditto, (\ref{eq:h_max}) \\
 $(V_h)_{h\geq 0}$ & upper bound on maximum variation of $f$ in a cell at level $h$ & \ditto, Claim~\ref{claim:doo_wh}\\
 & & \\
 $\LL_t^{rel}$ & relevant leaf nodes & Section~\ref{subsubsec:contextual_gp_algo} \\
 $\bar{x}_{h,i}^{(t)}$ & & \ditto \\
 $I_t^c$ & Index used for action selection in Algorithm~\ref{alg:contextual_tree} & \ditto, (\ref{eq:index_contextual_tree}) \\
 $\X_c$, $\X_a$ & Context space and Action space & Section~\ref{subsec:contextual_gp_bandits} \\
 & & \\ \hline
 
  \multicolumn{3}{|l|}{Parameters of Algorithm~\ref{alg:zoom} }\\ \hline 
  
 $A_t$ & Set of active points & Section~\ref{subsec:bayesian_zooming_algo}\\
 $r(x)$ & radius associated with a point $x \in A_t$ & \ditto \\
 $r_k$ & $diam(\X)2^{-k}$ & \ditto \\
 $W(r_k)$ & upper bound on variation of $f$ in $B(x,r_k,l)$ for any $x \in \X$ & Claim~\ref{claim:zoom_r_k_W_k} \\
 
 \hline
\end{tabular}
\label{table:table_of_symbols}

\caption{Table of symbols used in the paper}

\end{table}

\section{Algorithm for GP bandits}
\label{section:main_algorithm}

We begin this section by describing the general outline of all the algorithms proposed in this paper in Section~\ref{subsec:general_approach}. Then we 
introduce our tree based algorithm for GP bandits and obtain high probability bounds on its regret in Section~\ref{subsec:tree_based_algo}. 

\subsection{General approach}
\label{subsec:general_approach}

At any time $t$, we maintain a discretization (i.e., a finite subset) of $\X$, denoted by $\X_t$. To each $x \in \X_t$, we have an associated 
\emph{confidence region} denoted by $Reg_t(x)$, and an \emph{index} $Ind_t(x)$ which is a high probability upper bound on the maximum  value of the function
$f$ in $Reg_t(x)$. The index $Ind_t(x)$ depends on three quantities: (a) the actual function value at $x$, (b) the amount of uncertainty in the function 
value at $x$, and (c) the amount of variation in the function value in $Reg_t(x)$. 
We  proceed as follows:

\begin{itemize}
	\item In each round, we select a candidate point $x_t$ \emph{optimistically} by maximizing $Ind_t(x)$ over $\X_t$. 
	
	\item If the uncertainty in the function value at $x_t$ is smaller than the variation of $f$ in the confidence region, it means that we 
	must \emph{refine} our discretization in the confidence region associated with $x_t$. 
	
	\item If, on the other hand, the uncertainty in the function value at $x_t$ is larger than the variation of $f$ in the associated confidence 
	region, our algorithm evaluates the function at this point to reduce this uncertainty. 
\end{itemize}

In Section~\ref{subsec:tree_based_algo} we present an algorithm for GP bandits which uses a hierarchical partitioning scheme for locally refining the search 
space similar to \citep{munos2014bandits,bubeck2011x,wang2014bamsoo}. Alternatively, the \emph{covering oracle} based approach used by 
\citet{slivkins2014contextual,kleinberg2013bandits} can also be employed for refining the discretization, and we describe such an algorithm in Section~\ref{subsec:bayesian_zooming_algo}. We also apply this approach to design an adaptive algorithm for the Contextual GP bandits problem in Section~\ref{subsec:contextual_gp_bandits}.

	\subsection{Tree based Algorithm}
	\label{subsec:tree_based_algo}
	We now describe our  algorithm for GP bandits and derive high probability bounds on its regret. 
Our algorithm  is motivated by several tree based methods that have been  proposed for function optimization under Lipschitz-like assumptions, such as \citep{bubeck2011x,munos2011doo,munos2014bandits}. Assuming that the metric space $(\X,l)$ is \emph{well behaved}, i.e., we have a sequence of subsets $(\X_h)_{h\geq 0}$ whose associated cells form a tree of partitions of $\X$, we proceed as follows:
\begin{itemize}
\item In every round $t$, the algorithm maintains an active set of leaf nodes denoted by $\LL_t$, such that the cells of the nodes in $\LL_t$ partition $\X$. This active set is initialized to $\LL_0 = \{x_{0,1}\}$ with the associated cell $\X_{0,1} = \X$. 

\item The algorithm selects a node from $\LL_t$ by maximizing an index $I_t$. Then index $I_t(x_{h,i})$ is an upper confidence bound (UCB) on the maximum function value in cell $\X_{h,i}$ and is defined as 
\begin{equation}
\label{eq:ucb}
I_t(x_{h,i}) = \bar{U}_t(x_{h,i}) +  V_h.
\end{equation}
The term $\bar{U}_t(x_{h,i})$ in the above equation is a high probability upper bound on the function value at $x_{h,i}$ and is defined as 
\begin{equation}
\label{eq:function_upper_bound}
\bar{U}_t(x_{h,i}) = \min\big( \mu_{t-1}(x_{h,i})  + \beta_n\sigma_{t-1}(x_{h,i}), \mu_{t-1}(p(x_{h,i})) + \beta_n\sigma_{t-1}(p(x_{h,i})) + V_{h-1} \big) \end{equation}
where $p(x_{h,i})$ is the parent node of $x_{h,i}$. 
For any $h\geq 0$, the term $V_h$ is an upper bound on the maximum function variation in any cell $\X_{h,i}$ at level $h$. Thus, we see that $\bar{U}_t(x_{h,i})$ computes an upper bound on the value of $f(x_{h,i})$ in two ways and takes their minimum, while adding $V_h$ to it gives us an upper bound on the maximum function value in the cell $\X_{h,i}$.

\item Having chosen the point ($x_{h_t,i_t}$) according to the selection rule (Line-2 of Algorithm~\ref{alg:gp_tree}) we take one of the following two actions :
	\begin{itemize}
    \item \emph{Refine:} If $\beta_n\sigma_{t-1}(x_{h_t,i_t}) \leq V_{h}$, then the node $x_{h_t,i_t}$ is expanded, i.e., the $N$ children nodes $\{x_{h_t+1,j}: N(i_t-1)+1 \leq j\leq Ni_t \}$ of the node $x_{h_t,i_t}$ are added to the set of leaves, and $x_{h_t,i_t}$ is removed from it. (Lines 4-5 of Algorithm~\ref{alg:gp_tree})
    
		\item \emph{Evaluate:} Otherwise, then the function is evaluated at the point $x_{h_t,i_t}$, i.e., we observe the noisy function value $y_t = f(x_{h_t,i_t}) + \eta_t$ and update the posterior distribution of $f$. (Lines 7-9 of Algorithm~\ref{alg:gp_tree})

	\end{itemize}

\end{itemize}

The steps of the algorithm are shown as a pseudo-code in Algorithm~\ref{alg:gp_tree}. The algorithm maintains two counters, $t$ which counts the total number of
function evaluations and refinements, and $n_e$ which keeps track of the number of function evaluations. The algorithm stops after $n$ function evaluations,
and recommends a point from one of the deepest expanded cells (for minimizing $\mathcal{S}_n$).
The second condition on Line~3 of Algorithm~\ref{alg:gp_tree} is added to prevent the (unlikely) scenario in which the algorithm keeps refining indefinitely without evaluating the function.

\begin{algorithm}[ht!]
		\label{alg:gp_tree}
		\DontPrintSemicolon
		\SetAlgoLined
		\SetKwInOut{Input}{Input}
        \SetKwInOut{Output}{Output}
        \SetKwInOut{Initialize}{Initialize}
		\Input{$n>0$, ($\X_h)_{h\geq0}$, $\beta_n$,  $(V_h)_{h\geq 0}$, $h_{\max}$}
		\BlankLine
		\Initialize{$\LL_0 = \{ x_{0,1} \}$, $t=1$, $n_e=0$  }
        		\BlankLine

		\While{$n_e \leq n$}{
			choose $x_{h_t,i_t} = \argmax_{x_i \in \LL_t} I_t(x_{h,i})$  \;
			
			\uIf{ $\beta_n \sigma_{t-1}(x_{h_t,i_t}) \geq V_h$ AND $h_t\leq h_{\max}$ }{
            $\LL_{t+1} = \LL_t \setminus \{ x_{h_t,i_t}  \}$ \; 
				$\LL_{t+1} = \LL_{t+1} \cup \{ x_{h_t+1,j} | N(i-1)+1 \leq j \leq Ni  \}$ \;				
			}
			\Else{
				$y_t = f(x_{h_t,i_t}) + \eta_t$ \;
                update posterior $\mu_t(x)$ and $\sigma_t(x)$ \;
				$n_e \leftarrow n_e+1$ \;

			}	
            $t \leftarrow t+1$ \;
		}
		\Output{ $x(n)$: the deepest expanded node }	
		\caption{Tree based Algorithm for GP bandits}
	\end{algorithm}

\begin{remark}
 \label{remark:doubling}
 The parameter $\beta_n$ of Algorithm~\ref{alg:gp_tree} requires the knowledge of the horizon or the budget $n$. However, we can use the well known \emph{doubling trick}\citep[Section 2.3]{cesa2006prediction} to make our algorithm 
 \emph{anytime} without any change in the theoretical regret guarantees. The trick is to  work in phases of exponentially increasing lengths, and applying the algorithm with known horizon (equal to the duration of the phase)
 in each phase. 
\end{remark}

\subsection{Analysis of Algorithm~\ref{alg:gp_tree}}
\label{subsec:regret_analysis_gp_bandits}

In this section, we first specify the assumptions on the  covariance functions required for the theoretical analysis and then furnish the missing details of our tree based algorithm for GP bandits. Finally, we derive high probability bounds on the cumulative and simple regret for our algorithm. 

\subsubsection{Assumptions on the covariance functions}
\label{subsubsec:assumptions_on_K}
 To analyze our algorithm, we will restrict our attention to  a class of covariance functions, denoted by $\mathcal{K}$, such that for any $K \in \mathcal{K}$, we have:
	\begin{enumerate}[label=\textbf{A\arabic*}]
%		\item For any $x,y \in \X$,  $K(x,y) = K(\|x-y\|)$.
		\item For any  $x,y \in \X$, we have $d(x,y) \leq g(l(x,y))$ for some non-decreasing continuous function $g:\mathbb{R}^+\rightarrow \mathbb{R}^+$, such that ${g(0)=0}$. Recall that $l$ is assumed to be any metric on the space $\X$, and $d$ is the natural metric induced on $\X$ by the zero mean GP with covariance function $K$. 
        
        \item Moreover, we require that there exists a $\delta_K>0$ such that for all $r \leq \delta_K$, we have for constants $C_K>0$ and $0<\alpha \leq 1$ satisfying
		
		\begin{equation}
		\label{eq:assump2}
		g(r) \leq C_Kr^{\alpha}.
		\end{equation}
		
	\end{enumerate}

Assumption $A2$ informally requires that at least for small distances, points which are close in the metric $l$ are also close in $d$. These assumptions are satisfied by all the commonly used kernels such as squared exponential (SE), and the Mat\'{e}rn family of kernels. It also includes other kernels such as $K(r) = \max(0,1-r)$  and the rational quadratic kernel $K(r) = (1 + c_1r^2)^{-c_2}$ for some $c_1,c_2>0$.

	\begin{remark} We note that $\mathcal{K}$  is closed under finite addition and multiplication operations. This is an important property as in many practical applications, often more than one kernels are combined through addition or multiplication to provide more accurate models \citep[Chapter-2]{duvenaud2014kernel},\citep{rasmussen2006gaussian}. 
\end{remark}
\begin{remark}
Assumption $A2$ implies that  if $(\X,l)$ has a finite metric dimension $D_1$, then the metric space $(\X,d)$ (where $d$ is defined 
in (\ref{eq:gp_metric})) has a finite metric dimension $D_1' = D_1/\alpha$. This fact is used in  Proposition~\ref{prop1} in Section~\ref{section:technical_results}.  
	\end{remark}

\subsubsection{Details of the algorithm}
\label{subsubsec:details_of_algo1}
To complete the description of Algorithm~\ref{alg:gp_tree}, we need to specify the choice of the parameters $h_{\max}$, $\beta_n$, and $(V_h)_{h\geq 0}$. 

First we observe that for all $t$, we have $|\LL_t|\leq M(\X,v_2\rho^{h_{\max}},l)$. This follows from the assumption $P3$ in Definition~\ref{def:well_behaved}. From the definition of metric dimension we can upper bound $M(\X,v_2\rho^{h_{\max}},l)$ by $C\rho^{-D_1h_{\max}}$. As will be evident in the proof of Theorem~\ref{theorem:regret_gp_bandits}, an appropriate choice of the parameter $h_{\max}$ is:
\begin{equation}
 \label{eq:h_max}
 h_{\max} = \frac{\log n}{2\alpha \log(1/\rho)}\big(1+1/\alpha\big).
\end{equation}

\begin{claim}
		\label{claim:doo_Bn}
		With  $\beta_n  = \Oh(\sqrt{\log(n) + u})$, the following event occurs  with probability at least $1-e^{-u}$ for any $u>0$:  
		\begin{equation}
		 \label{eq:event_Omega_u5}
		\Omega_{u5}=	\{ \forall 1 \leq t \leq t_n, \forall x \in \LL_t: |f(x) - \mu_{t-1}(x)| \leq \beta_n\sigma_{t-1}(x)  \}.
		\end{equation}
		where $t_n$ is the (random) number of rounds required by the algorithm to complete $n$ function evaluations. 
	\end{claim}
	
	\begin{proof}
	 The largest value that the random variable $t_n$ can take is $h_{\max}n$, and for any $t\leq t_n$ we have $|\LL_t| \leq C\rho^{-D_1h_{\max}}$. Based on these two observations, we can claim the following:
	 \begin{align*}
	1-Pr(\Omega_{u5}) =  & Pr( \exists t \leq t_n, \exists x_{h,i} \in \LL_t: |\mu_{t-1}(x_{h,i}-f(x_{h,i}| > \beta_n\sigma_{t-1}(x_{h,i}) \\
	  \leq & \sum_{t=1}^{t_n}
	  \sum_{x_{h,i}\in \LL_t} Pr( |\mu_{t-1}(x_{h,i}-f(x_{h,i}|> \beta_n\sigma_{t-1}(x_{h,i}) \\
	  \leq & \sum_{t=1}^{t_n} \sum_{x_{h,i}\in \LL_t} 2 e^{-\beta_n^2/2} \\
	  \leq & 2(h_{\max}n)(C\rho^{-D_1h_{\max}})e^{-\beta_n^2/2} 
	 \end{align*}
Finally, we get the required bound by selecting $\beta_n^2 = \Oh\big(u + 2\log(h_{\max}n) + D_1h_{\max}\log(1/\rho) \big)$.
	 
	\end{proof}

\begin{remark}
The calculation of $\beta_n$ above is based on the worst case assumption that $|\LL_t| = M(\X,v_2\rho^{h_{\max}},l)$. In the case of $\X \subset \mathbb{R}^D$ and for odd values of $N$, we can use a tighter bound $|\LL_t| \leq nNh_{\max}$ which gives us $\beta_n = \Oh( \sqrt{u + \log(nh_{\max})})$ which allows us to consider larger values of $h_{max}$.
\end{remark}

	Next, we obtain the expressions for the parameters $(V_h)_{h\geq 0}$ as an immediate consequence of Corollary~\ref{cor:chaining}:
	\begin{claim}
		\label{claim:doo_wh}
		Suppose the metric space $(\X,l)$ is well-behaved in the sense of Definition~\ref{def:well_behaved} with subsets $(\X_h)_{h\geq 0}$ and associated parameters $v_1,v_2$ and $\rho$.   Let us define the event $\Omega_{u6}$ as
		\begin{equation}
		\label{eq:event_Omega_u6}
		\Omega_{u6} = \{ \forall h \geq 0; \forall 1\leq i \leq N^h: \sup_{x \in B(x_{h,i},v_1\rho^h,l)}|f(x) - f(x_{h,i})| \leq V_{h} \}.
		\end{equation}
		Then for the choice of $V_h$
		\[
		V_h = 4g(v_1\rho^h)(\sqrt{2u + C_4 + h\log N + 2D_1 \log(1/g(v_1\rho^h))} + C_3)
		\]	we have $Pr(\Omega_{u6}) \geq 1-e^{-u}$ for any $u>0$. Here $C_3$ and $C_4$ are the  positive constants defined in Corollary~\ref{cor:chaining}. 
	\end{claim}

\subsubsection{Regret Bounds}
\label{subsubsec:regret_bounds_algo1}
Before presenting the regret bounds, we first characterize the sub-optimality as well as the number of times points are evaluated by Algorithm~\ref{alg:gp_tree}.

	\begin{lemma} 
	\label{lemma:tree_lemma1}
	Under  the events $\Omega_{u5}$ (\ref{eq:event_Omega_u5}), and $\Omega_{u6}$ (\ref{eq:event_Omega_u6}), the following statements are true:
		\begin{itemize}
        
        \item If at time $t$ a point $x_{h_t,i_t}$ is evaluated by the algorithm, then the suboptimality of the selected point (denoted by $\Delta(x_{h_t,i_t}))$ can be upper bounded using $V_{h_t}$: 
       \begin{equation}
\label{eq:delta_1}
       \Delta(x_{h_t,i_t}) \coloneqq f(x^*) - f(x_{h_t,i_t}) \leq (2N+1)V_{h_t}.        
       \end{equation}
   \item Furthermore, if the evaluated point $x_{h_t,i_t}$ satisfies the condition that $h_t < h_{\max}$, then we have another bound on $\Delta(x_{h_t,i_t})$ in terms of the posterior variance: 
	\begin{equation}
\label{eq:delta_2}
	\Delta(x_{h_t,i_t}) \leq 3\beta_n\sigma_{t-1}(x_{h_t,i_t}).
	\end{equation}

            \item A point $x_{h,i}$, with $h < h_{\max}$, may be evaluated no more than $q_h$ times before it is expanded, where 
            where \[
            q_h = \frac{\sigma^2\beta_n^2}{V_h^2}.
            \]
            Furthermore for $h$ large enough so that $v_1\rho^h \leq \Delta_K$, we have 
\[
q_h \leq \frac{\sigma^2\beta_n^2}{g(v_1\rho^h)^2C_3}
\]
using the assumptions on the covariance function $K$. 			
	
		\end{itemize}
		
	\end{lemma}

\begin{proof}

We recall that under the event $\Omega_{u5}$ we have $|f(x_{h,i}) - \mu_{t-1}(x_{h,i})| \leq \beta_n\sigma_{t-1}(x_{h,i})$ for all $x_{h,i} \in \LL_t$ and for all $t\geq 1$. Furthermore, form the definition of event $\Omega_{u6}$, we have the following for all $h\geq 0$ and $1\leq i \leq N^h$:
\[
\sup_{x_1,x_2 \in \X_{h,i}} |f(x_1)-f(x_2)| \leq V_h.
\] 
Using these two facts we can prove the first part of this lemma in the following way:

\begin{itemize}
\item Suppose at time $t$, the true maximizer $x^*$ lies in the cell $\X_{h_t^*,i_t*}$ associated with the point $x_{h_t^*,i_t^*}$, and the algorithm selects and  evaluates the point $x_{h_t,i_t}$. Then we have the following sequence of inequalities:
\begin{align*}
f(x^*) \leq I_t(x_{h_t^*,i_t^*}) & \leq I_t(x_{h_t,i_t}) = \bar{U}_t(x_{h_t,i_t}) + V_{h_t} \\
& \stackrel{(a)}{\leq} \mu_{t-1}(p(x_{h_t,i_t})) + \beta_n\sigma_{t-1}(p(x_{h_t,i_t})) + V_{h_t-1} + V_{h_t} \\
& \stackrel{(b)}{\leq} f(p(x_{h_t,i_t})) + 2\beta_n\sigma_{t-1}(p(x_{h_t,i_t})) +  V_{h_t-1} + V_{h_t} \\
& \stackrel{(c)}{\leq} \big(f(p(x_{h_t,i_t})) + V_{h_t-1}\big)+  2V_{h_t-1} + V_{h_t} \\
& \stackrel{(d)}{\leq} f(x_{h_t,i_t}) + 2V_{h_t-1} + V_{h_t}\\
& \stackrel{(e)}{\leq} f(x_{h_t,i_t}) + (2N+1)V_{h_t}
\end{align*}

The inequality $(a)$ follows from the definition of $\bar{U}_t(x_{h_t,i_t})$, while $(b)$ uses the fact that $f(p(x_{h_t,i_t})) \geq \mu_{t-1}(p(x_{h_t,i_t})) - \beta_n\sigma_{t-1}(p(x_{h_t,i_t})$ under event $\Omega_{u6}$. For $(c)$, we use the fact that $p(x_{h_t,i_t})$ must have been expanded which means $\beta_n\sigma_{t-1}(p(x_{h_t,i_t}))$ must be smaller than $V_{h_t-1}$. For inequality $(d)$ we observe that $x_{h_t,i_t}$ must lie in the cell associated with $p(x_{h_t,i_t})$ and then use the definition of $V_{h_t-1}$, while $(e)$ follows from the triangle inequality. 

\item 
For obtaining the bound in (\ref{eq:delta_2}),  we again use the definition of $\bar{U}_t(x_{h_t,i_t})$ to now upper bound it by the other term in its definition to get: 
\begin{align*}
f(x^*) \leq I_t(x_{h_t^*,i_t^*}) & \leq I_t(x_{h_t,i_t}) = \bar{U}_t(x_{h_t,i_t}) + V_{h_t} \\
& \leq \mu_{t-1}(x_{h_t,i_t}) + \beta_n\sigma_{t-1}(x_{h_t,i_t}) + V_{h_t} \\
& \leq f(x_{h_t,i_t}) + 2\beta_n\sigma_{t-1}(x_{h_t,i_t})  + V_{h_t} \\
& \stackrel{(f)}{\leq} f(x_{h_t,i_t})+ 3\beta_n\sigma_{t-1}(x_{h_t,i_t}) 
\end{align*}
The inequality $(f)$ above uses the fact that since the function is evaluated at time $t$, we must have $\beta_n\sigma_{t-1}(x_{h_t,i_t}) \geq V_{h_t}$.

\item A point $x_{h,i}$ must be evaluated by the algorithm sufficiently many times to reduce the uncertainty in the function value at $x_{h,i}$ from below $V_{h-1}$ to below $V_h$. We provide a  loose upper bound on this 
quantity, by providing an upper bound on the number of function evaluations sufficient to reduce the uncertainty in  the value of $f(x_{h,i})$ to below $V_h$. Using the first part of Proposition~\ref{prop:posterior_variance},
we define $q_h$ as follows to get the required result. 
\[
q_h = \min \{m: \beta_n \frac{\sigma}{\sqrt{m}} \leq V_h \}
\] 
\end{itemize}

\end{proof}

\begin{remark}
\label{remark:near_optimality_dimension}
From Lemma~\ref{lemma:tree_lemma1}, we can see that the algorithm only selects points lying in $\X_{(2N+1)V_h} = \{x \in \X: f(x) \geq f(x^*)-(2N+1)V_h \}$ for $h\geq 0$. 
Now, for $r\leq \delta_K$, let us define $\zeta_K(r)=V_{h_r}$ where $h_r = \min \{h\geq 0: v_1\rho^h \geq r \}$  and $\tilde{D}=D^f(\delta_K,\zeta_K)$ where the term $D^f(\cdot,\cdot)$ was introduced in Definition~\ref{def:near_opt}.  
We will use this term $\tilde{D}$ for presenting  our regret bounds, and will refer to it as the near optimality dimension of $\X$ associated with the function $f$. 
\end{remark}

We can now state the main result of this section which gives us high probability bounds on the cumulative as well as simple regret of Algorithm~\ref{alg:gp_tree}. 

\begin{theorem}
\label{theorem:regret_gp_bandits}
Suppose the unknown function $f$ is a sample from a $GP(0,K)$, with $K \in \mathcal{K}$ and $\X$ is a well behaved  metric space (in the sense of Definition~\ref{def:well_behaved})
with finite metric dimension (see Definition~\ref{def:near_opt}) $D_1$. 

For any $u>0$, the following bounds are true with probability at least $1-2e^{-u}$ for Algorithm~\ref{alg:gp_tree}: 
\begin{itemize}
\item The cumulative regret incurred by Algorithm~\ref{alg:gp_tree} satisfies

\begin{equation}
\mathcal{R}_n \leq \tilde{\Oh}(n^{1 - \frac{\alpha}{2\alpha + \tilde{D}}}) \label{eq:gp_bandits_cr_dtype},
\end{equation}
where $\tilde{D}$ (described in Remark~\ref{remark:near_optimality_dimension}) is a non-negative random variable always less than or equal to $D_1$.

\item Furthermore, if we make the assumption that $K(x,x) \leq 1$ for all $x \in \X$, we  have an information type bound on the cumulative regret:
\begin{equation}
\mathcal{R}_n \leq \Oh( \sqrt{n\gamma_n \log(n)}).
\label{eq:gp_bandits_cr_itype}
\end{equation}

\item Finally, we also have an upper bound on the simple regret:
\begin{equation}
\mathcal{S}_n \leq \tilde{\Oh}(n^{-\frac{\alpha}{2\alpha + \tilde{D}}}). \label{eq:gp_bandits_simple}
\end{equation}
\end{itemize}
\end{theorem}

The proof of this result is given
in Appendix~\ref{appendix:gp_tree_regret}. 

\begin{remark}
The bounds in (\ref{eq:gp_bandits_cr_dtype}) and (\ref{eq:gp_bandits_simple}) which depend on  near-optimality dimension   will be referred to as 
\emph{dimension-type} regret bounds in accordance with the terminology used by \cite{slivkins2014contextual}. We note that since the cumulative regret
of the algorithm can be bounded in two ways, by taking the minimum of the bounds in (\ref{eq:gp_bandits_cr_dtype}) and (\ref{eq:gp_bandits_cr_itype}), we can get a uniformly better upper bound on the 
cumulative regret for our algorithms for all GPs with admissible covariance functions with $K(x,x) \leq 1$. 
\end{remark}

\section{Discussion}
\label{section:discussion}

The analysis of Algorithm~\ref{alg:gp_tree} presented in the previous section is valid for arbitrary well-behaved search space $\X$, any covariance function $K\in \K$ and in the presence of observation noise.
In this section, we discuss the performance of 
our algorithm under some specific problem instances. In particular, we first show that our adaptive approach leads to  computational requirements which do not explode with the dimension $D$ when $\X \subset \mathbb{R}^D$, unlike
the existing algorithms for GP bandits. 
We then validate the intuition provided by our  toy examples in Section~\ref{subsec:toy_examples} by showing that the information-type bounds are indeed loose for an important family of Gaussian 
Processes. 
Finally, we specialize our results to the noiseless case, and show that our algorithm compares favorably with BaMSOO in the pure exploration problem.

%--------------------------------------------------------
%COMPUTATIONAL BENEFITS OF ADAPTIVITY
%--------------------------------------------------------

\subsection{Computational benefits of adaptivity }
\label{subsec:computational_benefits}

As an upshot of the adaptive discretization of the search space, the computational complexity of Algorithm~\ref{alg:gp_tree} does not grow exponentially with the dimension of the search space, as shown in the following result.

\begin{claim}
\label{claim:comptuationaL_benefits}
If $\X = [a,b]^D$ for $a,b\in \mathbb{R}$ and $D<\infty$, and $N$ is an odd positive integer, then the computational complexity of running Algorithm~\ref{alg:gp_tree} with a budget of $n$ evaluations is $\mathcal{O}(h_{\max}n^{4}) = \Oh(n^4\log n)$ for all values of $D$. 
\end{claim}

\begin{proof}
Recall that the search space considered here is well-behaved in the sense of Definition~\ref{def:well_behaved}, and has a finite metric dimension $D_1=D$. Furthermore, since $N$ is odd, we observe that the sequence of partitions $(\X_h)_{h \geq 0}$ are nested. More specifically, if the cell associated with a node $x_{h,i}$ is refined to add the nodes $\{ x_{h+1,j}; (N-1)i+1 \leq j \leq Ni \}$ to the leaf set, then we have $x_{h+1,(N-1)i+(N+1)/2} = x_{h,i}$. 

Let $t_n$ denote the number of rounds required for $n$ function evaluations by the algorithm,  and let $(\tau_j)_{j=1}^n$ denote the round numbers in which function evaluations are performed. Now, if we define $\tau_0=1$, then we claim that the following:
\begin{itemize}
\item The posterior distribution is recomputed in rounds $(\tau_j+1)_{j=0}^n$ based on the observations. The computational task of  of updating the posterior based on $j$ observations in the round $\tau_j+1$ can be performed in  $\Oh(j^2)$ operations by using the  Cholesky Decomposition. Thus the 
total cost for posterior computation is $\Oh(n^3)$. 

\item For all $t$ such that $\tau_j+1 < t \leq \tau_{j+1}$, the index $I_t$ at a given point can be computed in $\Oh(j^2)$ operations. Since every refinement step adds $N-1$ new points to the leaf set and $\tau_{j+1}-\tau_j \leq h_{\max}$ for all $j\geq 0$, the total cost of computing the index in this time interval is $\Oh((N-1)h_{\max}j^2)$. For $t \in \{\tau_j+1; 0\leq j\leq n\}$, the index must be recomputed for the entire leaf set $\LL_t$ whose cardinality is upper bounded by $(N-1)h_{\max}j$, and thus the computational cost is $\Oh(j^2(N-1)h_{\max}j) = \Oh( h_{\max}j^3)$.  Thus the total cost of computing the index $I_t$ for all $t\leq t_n$ is $\Oh( (N-1)h_{\max}n^4)$. 

\item For selecting the candidate points $x_{h_t,i_t}$ for $t \in \{ \tau_{j}+1, 0 \leq j < n\}$, we need to perform an exhaustive search over the entire leaf set $\LL_t$ which is a $\Oh( (N-1)h_{\max}j)$ operation. At all other times, we only need to search over the $(N-1)$ new descendants of the previous candidate point. Thus the total cost of selecting $x_{h_t,i_t}$ for $1\leq t \leq t_n$ is $\Oh( (N-1)h_{\max}n^2 + (N-1)h_{\max}n) = \Oh((N-1)h_{\max}n^2)$. 

\item As mentioned earlier, the refinement of a cell $\X_{h,i}$ when $\X \subset \mathbb{R}^D$ is performed by dividing it equally in $N$ parts along its longest side. This requires $\Oh(DN)$ operations, so the total cost of refining the search space is $\Oh(h_{\max}nDN)$. 

\end{itemize}
Thus the overall computational cost of running the algorithm with a budget of $n$ function evaluations for fixed $D$ and $N$ is $\Oh( h_{\max}n^4)$, which is equal to $\Oh(n^4\log n)$ using the constraint on $h_{\max}$ given in (\ref{eq:h_max}).  

\end{proof}

As shown above, the computational complexity of our algorithm scales linearly with the dimension of the search space. This is in contrast to the existing algorithms for GP bandits which perform a global maximization of an \emph{acquisition function} ($\psi_t(\cdot)$) for selecting a query point:
	\[
	x_t \in \argmax_{x \in \X} \psi_t(x)
	\]
The computational cost of performing this operation exactly can be exponential in $D$. For example in the GP-UCB algorithm the acquisition function is the upper confidence bound at each point $x \in\X$. Over a search space $\X \subset \mathbb{R}^D$, for the theoretical results to be valid, the GP-UCB algorithm must select a query point at time $t$ by calculating and then maximizing the UCB over a uniform grid  of size $\Oh(t^{2D})$ \citep{srinivas2012information}. Thus the overall computational cost of running this algorithm for $n$ rounds is $\Oh\big(\sum_{t=1}^n t^{2D+2} \big) = \Oh\big(n^{2D+3} \big)$.

%--------------------------------------------------------
%IMPROVED REGRET BOUNDS FOR MATERN KERNELS
%--------------------------------------------------------

\subsection{Improved bounds for Mat\'ern kernels}
\label{subsec:improved_bounds_for_matern}
Mat\'{e}rn  kernels are a widely used class of kernels  parameterized by a smoothness parameter $\nu$. For half integer values of $\nu = m + 1/2$, the Mat\'{e}rn kernels have the form:
	\[
	K(r) = K(0)(1 + p_m(r))e^{-c_1\sqrt{\nu}r}
	\] where $p_m = \sum_{j=1}^{m}a_ir^i$ for some $a_i >0$ for all $1\leq i \leq m$. 
	Thus we can write for any $x,y \in \X$ such that $l(x,y) = r$:
	\begin{align*}
	d(x,y) &= [2K(0)(1- (1+p_m(r))e^{-c_1\sqrt{\nu}r})]^{1/2} \\
% 	& \leq [2K(0)(1- e^{-c_1\sqrt{\nu}r})]^{1/2} \\
% 	&\leq [2K(0)c_1\sqrt{\nu}r]^{1/2}
&\leq C_K(\nu)r^{\alpha}
	\end{align*}
It is easy to check that for $\nu=1/2$, we have $\alpha=1/2$, and for all other half-integer values of $\nu$, we have $\alpha=1$.    
	So, for Mat\'{e}rn kernels, our algorithm has a dimension-type upper bound on regret of the form $\tilde{\Oh}(n^{(\tilde{D}+\alpha)/(\tilde{D}+2\alpha)})$ for all $\nu=m+1/2$ with $m \geq 0$ and
	$\alpha \in \{1/2,1\}$.  This improves upon the existing upper bounds on Mat\'{e}rn kernels in the following two ways (since the existing bounds are true only when $\X \subset
	\mathbb{R}^D$, we will restrict our comparison to this case, and so we have $D_1 = D$ here):
	\begin{itemize}
		\item The existing regret bounds are only valid for the case of $\nu>1$ \citep{srinivas2012information,contal2016stochastic}, whereas the dimension-type regret bounds of our 
		algorithm is valid for all $\nu\geq 1/2$. In particular, for the exponential kernel ($\nu=1/2$, also referred to as the Ornstein-Uhlenbeck process), \cite{srinivas2012information} conjectured that it may not be
		possible to derive a regret bound of the form shown in (\ref{eq:information_type}). This conjecture was refuted by \cite{contal2016stochastic}, but 
		the authors did not provide an explicit characterization of $\mathcal{R}_n$ as no suitable bounds for $\gamma_n$ for this kernel are known. Our result provides
		an upper bound on the cumulative regret for the exponential kernel of the form $\mathcal{R}_n \leq \tilde{\Oh}(n^{(2\tilde{D}+1)/(2\tilde{D}+2)})$, which is,
		to the best of our knowledge, the first explicit sublinear bound on the cumulative regret for the  GP bandits problem with exponential kernel. 
		
		\item The existing regret bounds for Mat\'{e}rn kernels have the form $\tilde{\mathcal{O}}\big(n^{\frac{D(D+1)+\nu}{D(D+1)+2\nu}} \big) $\citep{srinivas2012information,contal2016stochastic} for $\nu>1$. As compared to this, the bounds  obtained by our algorithm,  after substituting $\alpha=1$ for Mat\'{e}rn kernels with $\nu >1$ depend upon  $\tilde{D}$, which itself is a random variable dependent on the  sample function $f$ of the Gaussian Process and can take values anywhere from $0$ to $D$. Assuming the worst case value of $\tilde{D}=D$, we observe that for $ D \geq \nu-1$, we have
		\[
		\frac{D+1}{D+2} \leq \frac{D(D+1)+\nu}{D(D+1)+2\nu}
		\]
		Thus $D \geq \nu -1$ is a sufficient condition for our upper bounds to be tighter than the best known bounds for Mat\'{e}rn kernels. The two most commonly used Mat\'{e}rn kernels in Machine learning correspond to $\nu=3/2$ and $\nu=5/2$ \citep[Chapter~4]{rasmussen2006gaussian}, for which the sufficient condition reduces to $D\geq 1$ and $D\geq 2$ respectively. 
	\end{itemize}

%--------------------------------------------------
%REGRET UNDER NOISELESS OBERSVATIONS
%--------------------------------------------------------

\subsection{Regret under noiseless observations}
 \label{subsec:noiseless_observations}

 In this section, we consider the special case where there is no observation noise, and specialize the regret bounds of our algorithm to this setting.
In particular we have the following bounds:
\begin{claim}
 \label{claim:noiseless}
If in addition to the assumptions of Theorem~\ref{theorem:regret_gp_bandits}, we further assume that the observations are noiseless, i.e., $\sigma = 0$, we get with high probability, the bounds 
\begin{align}
\mathcal{R}_n &\leq \tilde{\Oh}(n^{1-\frac{\alpha}{\tilde{D}}}) \label{eq:noiseless_cr1} \\
\mathcal{S}_n &\leq \tilde{\Oh}(n^{-\alpha/\tilde{D}})
\end{align}
if $\tilde{D} >0$, and 
\begin{align}
\mathcal{R}_n &\leq \tilde{\Oh}(1) \label{eq:noiseless_cr2} \\
\mathcal{S}_n &\leq \tilde{\Oh}(e^{-c_1\log(1/\rho)n}) \label{eq:exponentially_decaying_regret}
\end{align}
if $\tilde{D}=0$ and $h_{\max} = \Omega(n)$, for some constant $c_1>0$. 
 \end{claim}

\begin{remark}
We note that unlike Theorem~\ref{theorem:regret_gp_bandits}, we do not present information-type bounds on the cumulative regret in Claim~\ref{claim:noiseless}. This is because the information-type bounds given by (\ref{eq:information_type}) are not directly applicable in the noiseless setting as the term $\gamma_n$ becomes undefined for $\sigma=0$. 
\end{remark}

 As mentioned earlier, our work is motivated by BaMSOO, an adaptive algorithm for the Bayesian optimization problem which works only with noiseless observations \citep{wang2014bamsoo}.  BaMSOO builds upon the \emph{Simultaneous Optimistic Optimization}(SOO) algorithm of \cite{munos2011doo} by making the further assumption that  the unknown function is a sample from a GP, and then utilizes the posterior confidence intervals in selection of points. \cite{wang2014bamsoo} obtained an upper bound on the simple regret of the order $\tilde{\mathcal{O}}(n^{-c/D})$ for some $c>0$ which is similar to our simple regret bound in Claim~\ref{claim:noiseless}. However, our approach extracts more information about the function from the GP prior and has some advantages over BaMSOO in the pure exploration setting. 
 In particular, the derivation of regret bounds for BaMSOO required the assumption \citep[Assumption~2]{wang2014bamsoo} that the unknown function is approximately quadratic in the region around the maximum $x^*$, which for example is ensured if the covariance function has continuous partial derivative of order 6. Our result does not require this quadratic behavior, and is valid for kernels not satisfying the smoothness requirements, such as the exponential kernel $K(r) = ce^{-c_1r}$, and the  kernel $K(r) = (1-r)^+$. Furthermore, if for some instances of the function $f$, the random variable $\tilde{D}$ equals zero, then we obtain an exponentially decaying simple regret bound for Algorithm~\ref{alg:gp_tree}. This is unlike the simple regret bounds for BaMSOO which decay polynomially in $n$ for all admissible kernels. 
	
% \begin{itemize}

% \item 	The derivation of regret bounds for BaMSOO required the assumption \citep[Assumption~2]{wang2014bamsoo} that the unknown function is approximately quadratic in the region around the maximum $x^*$, which for example is ensured if the covariance function has continuous partial derivative of order 6. Our result does not require this quadratic behavior, and is valid for kernels not satisfying the smoothness requirements, such as the exponential kernel $K(r) = ce^{-c_1r}$, and the  kernel $K(r) = (1-r)^+$. 

% \item The simple regret bound on BaMSOO decays polynomially in $n$ for all admissible kernels. For our algorithm, however, as shown in (\ref{eq:exponentially_decaying_regret}), the simple regret can decay exponentially with $n$ when $\tilde{D} = 0$. The term $\tilde{D}$ is a random variable taking values in the range $[0,D_1]$ which has a complicated dependence on the sample $f$ of the GP. In Appendix~\ref{appendix:sufficient_condition}, we obtain sufficient conditions on the covariance function $K \in \K$ for which $\tilde{D}$ is zero under events of high probability. We note that the algorithm presented in \citep{kawaguchi2015bayesian} can attain exponential convergence rates under the extra assumption that the function satisfies a weak Lipschitz continuity assumption about its maximizer with respect to an unknown metric. However, this result is not directly comparable to the fully Bayesian setting considered in this paper. 

% \end{itemize}

\section{Extensions}
\label{section:extensions}

In this section, we first present an algorithm for GP bandits which uses an alternative approach to locally refining the search space as compared to Algorithm~\ref{alg:gp_tree}. While Algorithm~\ref{alg:gp_tree} requires a tree of partitions to adaptively discretize the space $\X$, the algorithm presented in Section~\ref{subsec:bayesian_zooming_algo} instead utilizes a \emph{covering oracle} to explore the search space. 

Next, in Section~\ref{subsec:contextual_gp_bandits} we  apply our general approach to design an adaptive algorithm for the problem of contextual GP bandits, an extension of the usual GP bandits problem first studied in \citep{krause2011contextual}.

\subsection{Bayesian Zooming Algorithm}
\label{subsec:bayesian_zooming_algo}

We now present a Bayesian version of the zooming algorithm for Lipschitz optimization introduced by \cite{kleinberg2013bandits} and analyze its regret. In particular, 
instead of assuming that the metric space $(\X,l)$ is well-behaved in the sense of Definition~\ref{def:well_behaved}, this algorithm requires  access to the space $(\X,l)$ through a \emph{covering oracle} (see
Remark~\ref{remark:covering_oracle} for definition) to locally refine the discretization.

The algorithm proceeds by constructing an increasing sequence of \emph{active} subsets of $\X$ denoted by $(A_t)_{t\geq 1}$. As with Algorithm~\ref{alg:gp_tree}, we can compute high probability upper and lower confidence intervals for the function values at points in $A_t$ for all $t\geq 1$. 
\begin{equation}
	\label{eq:zoom_interval}
	f(x) \in [\mu_{t-1}(x) - \beta_n\sigma_{t-1}(x), \mu_{t-1} + \beta_n\sigma_{t-1}(x) ] \hspace{2em} \text{w.h.p.}
	\end{equation} 
	for a suitable factor $\beta_n$. 
	
	Also, to each point $x$ that has been evaluated at least once, we assign a radius denoted by $r(x)$. The radius $r(x)$ can take values in a set $S=\{r_k=diam(\X)2^{-k}; k \in \N \} $, where \[ diam(\X) \coloneqq \sup_{x_1,x_2 \in \X} l(x_1,x_2)
    \] is the diameter of the metric space $(\X,l)$ and is assumed to be finite. 
% 	The radius values are allocated to a point $x$ in the active set by the algorithm in such a way to ensure that the uncertainty in the function value at that point is of the order of the variation of the GP sample within the ball $B(x,r(x),l)$. 
  For implementing the algorithm, we further require bounds $(W(r_k))_{k \in \N}$ such that for all $x \in \X$ and for all $k \in \N$, $W(r_k)$ is a bound on the variation of the GP sample
	in the ball $B(x,r_k,l)$ with high probability. We obtain these $W(r_k)$ using Proposition \ref{prop:chaining3}.  We also require a parameter $r_{\min}$ as input, which plays a role similar to $h_{\max}$ in Algorithm~\ref{alg:gp_tree}. The details behind the choice of these parameters are provided in Appendix~\ref{appendix:bayesian_zooming_algo}.

	Corresponding to each point that has been evaluated at least once, we have an associated confidence region $B(x,r(x),l)$, and furthermore, we also have an upper bound on the 
	maximum value of the function in that region (w.h.p.) given by the index:
	\begin{equation}
	\label{eq:zoom_idx}
	J_t(x) = \mu_{t-1}(x) + \beta_n\sigma_{t-1}(x) + W(r(x)).
	\end{equation}
    
In each round $t$, a candidate point is selected in an optimistic manner from the set $A_t$, i.e., 
	\begin{equation}
	\label{eq:zoom_selection}
	x_t = \argmax_{ x \in A_t}\hspace{1em} J_t(x).
	\end{equation}
	The index $J_t(x)$ can take a large values if :
	\begin{itemize}
		\item the point $x$ has been evaluated very few times, in which case the uncertainty at $x$ ( $\beta_n \sigma_{t-1}(x)$) as well as the bound on the variation of $f$ in the confidence region ($W(r(x))$) are large. 
 
		\item or if the point $x$ has been observed many times, and the true function value $f(x)$ is large. 
	\end{itemize}
In this way the selection rule strikes a balance between exploration of poorly understood regions, and exploitation of well explored regions with high function values. 
	
Having chosen a candidate point $x_t$ at time $t$, the algorithm takes one of two actions:
\begin{itemize}
\item \emph{Refine:} If the uncertainty in the function value a point $x_t$ is smaller than the bound on the variation of the function in the confidence region associated with point $x_t$, then the algorithm locally refines the search space, that is, it shrinks the radius of the confidence region associated with $x_t$ by a factor  of 2. 

\item \emph{Evaluate:} Otherwise, if the uncertainty in the function value is larger than the variation in the confidence region, the function is evaluated at the candidate point $x_t$. 
\end{itemize}

	In order to ensure that the entire search space is taken into consideration, the algorithm maintains at all times the following invariant:
	\begin{equation}
	\label{eq:zoom_invariant}
	\X \subset \cup_{ x \in A_t}B(x, r(x),l).
	\end{equation}
	If this invariant is violated, a point from the \emph{uncovered region} (i.e., $\X \setminus \cup_{ x \in A_t}B(x, r(x),l$) is added to the active set of points with an 
	associated radius $r_0 = diam(\X)$.

All the steps described above are formally stated as a pseudo-code in Algorithm~\ref{alg:zoom}.

	\begin{algorithm}
		\label{alg:zoom}
		\DontPrintSemicolon
		\SetAlgoLined
		\SetKwInOut{Input}{Input}\SetKwInOut{Output}{Output}
        \SetKwInOut{Initialize}{Initialize}
		\Input{$n>0$, $(r_k)_{k\geq 0}$, $(W(r_k))_{k\geq 0}$ , $r_{\min}$}
		\Initialize{$t=1$, $n_e=0$, $A_t = \{\}$ }
		\BlankLine
		
		\While{$n_e \leq n$}{
			choose $x_t = \argmax_{x_i \in A_t} \mu_{t-1}(x_i) + \beta_n\sigma_{t-1}(x_i) + W(r(x_i)))$ \;
     		\BlankLine
       
            \uIf{ \big($\beta_n\sigma_{t-1}(x_t) \leq W(r(x_t))$\big) AND \big( $r(x_t) \geq r_{\min}$  \big)} {
            $r(x_t) \leftarrow r(x_t)/2 $\;
            }
            		\BlankLine

            \Else{          
			evaluate $y_t = f(x_t) + \eta_t$ \;
			update posterior $\mu_t(x)$ and $\sigma_t(x)$ \;
			update $n_e \leftarrow n_e + 1$ \;
			}
            		\BlankLine

			\If{ $\X \not \subset \cup_{x_i \in A_t}B(x_i, r(x_i),l)$ } {
				Add a point $x \in \X \setminus\cup_{x_i \in A_t}B(x_i, r(x_i),l)$ to $A_{t}$, with $r(x) = r_0 = diam(\X)$. 
			}						
			$t \leftarrow t+1$ \;
    }	        
        \Output{$x(n)$: point with the smallest radius}
		\caption{Zooming Algorithm for GP bandits}
	\end{algorithm}

    \begin{remark}
     A key difference between Algorithm~\ref{alg:zoom} and the zooming algorithm for Lipschitz functions is that our algorithm only evaluates a point if the confidence radius associated with it is small enough (Lines 3-4 of Algorithm~\ref{alg:zoom}). This is unlike the zooming algorithm in \citep{kleinberg2013bandits}, in which a point is evaluated every round. This modification is necessary  to obtain the information type bounds on the cumulative regret for our algorithm. 
    \end{remark}
    
  	  \begin{remark}
  	  \label{remark:covering_oracle}
	 For maintaining the invariant described in (~\ref{eq:zoom_invariant}) and in Lines 10-12 of Algorithm ~\ref{alg:zoom}, we assume the existence of a \emph{covering oracle}
	 \citep[Section 1.5]{kleinberg2013bandits}, which takes in as inputs a finite set of balls and outputs whether these balls cover the entire space $\X$ or not. In the latter case, the covering oracle also returns an 
	 arbitrary point from the uncovered region of $\X$. 
    In our case, if at the beginning of round $t$ the entire space is covered by the balls (this is true at $t=2$), and suppose a point $x$ is selected by the algorithm and its confidence radius is shrunk
    from $r(x)$ to $r(x)/2$. Then at the beginning of the next round,  we only need to check whether the annular region
    $B(x,r(x),l)\setminus B(x,r(x)/2,l)$ is fully covered by the other balls or not. 
	\end{remark}

	Our next result shows that we can obtain the same regret performance for Algorithm~\ref{alg:zoom} as we did for the tree-based algorithm. 
	
\begin{theorem}
\label{theorem:regret_zoom}
Suppose the unknown function $f$ is a sample from a $GP(0,K)$, with $K \in \mathcal{K}$. 
 $(\X,l)$ is assumed to be a compact   metric space with finite metric dimension $D_1$ (see Definition~\ref{def:near_opt}). Moreover, 
we assume that we can access the metric space $(\X,l)$ through a \emph{covering oracle}. 

Then, for any $u>0$, the following bounds are true with probability at least $1-2e^{-u}$ for Algorithm~\ref{alg:zoom}: 
\begin{itemize}
\item We have the following dimension-type bound on the cumulative regret. 

\begin{equation}
\mathcal{R}_n \leq \tilde{\Oh}(n^{1 - \frac{\alpha}{2\alpha + \tilde{D}_Z}}), \label{eq:gp_bandits_cr_dtype}
\end{equation}
where $\tilde{D}_Z$ is the near-optimality dimension defined in Remark~\ref{remark:zoom_near_opt}

\item Under the extra assumption that $K(x,x) \leq 1$ for all $x \in \X$, we also have an information type bound on the cumulative regret:
\begin{equation}
\mathcal{R}_n \leq \Oh( \sqrt{n\gamma_n \log(n)}).
\label{eq:gp_bandits_cr_itype}
\end{equation}

\item Finally, we also have an upper bound on the simple regret:
\begin{equation}
\mathcal{S}_n \leq \tilde{\Oh}(n^{-\frac{\alpha}{2\alpha + \tilde{D}_Z}}). \label{eq:gp_bandits_simple}
\end{equation}
\end{itemize}
\end{theorem}

The details of the choice of the parameters of Algorithm~\ref{alg:zoom} as well as an outline of the proof of Theorem~\ref{theorem:regret_zoom} is provided in Appendix~\ref{appendix:bayesian_zooming_algo}. 

\begin{remark}
\label{remark:zoom_near_opt}
 The near-optimality dimension $\tilde{D}_Z$ used in the statement of Theorem~\ref{theorem:regret_zoom} can be defined similar to the definition of $\tilde{D}$ introduced in Remark~\ref{remark:near_optimality_dimension}. More 
 specifically, by Lemma~\ref{lemma:bayesian_zooming_algo} in Appendix~\ref{appendix:bayesian_zooming_algo} we know that Algorithm~\ref{alg:zoom} only selects evaluation points from sets of the form $\X_{5W(r_k)} = 
 \{ x \in \X : f(x^*) - f(x) \leq 5W(r_k) \}$ for $k\geq 0$. So we can proceed as in Remark~\ref{remark:near_optimality_dimension} to define $\tilde{D}_Z = D^f(\delta_K,\zeta_K)$, with $\zeta_K(z) \coloneqq 5W(r_{k_z})$ with 
 $k_z \coloneqq \min \{ k\geq 0: r_k \leq z \}$. 
\end{remark}

\subsection{Extension to Contextual GP bandits}
\label{subsec:contextual_gp_bandits}
The contextual bandit problem is a generalization of the multiarmed bandit (MAB) problem in which at the beginning of each round, the agent receives a context, and the task is to select an action which is optimal for the context 
received. \cite{krause2011contextual} considered this problem in the Bayesian framework with GP prior and proposed the CGP-UCB algorithm which is a variant of the GP-UCB algorithm. They obtained information-type regret bounds on the contextual regret for CGP-UCB and additionally, provided bounds on the maximum information gain ($\gamma_n$) for composite kernels over the product space. 
This problem has also been studied in the non-Bayesian setting by imposing Lipschitz condition on the payoff functions \citep{slivkins2014contextual}. 

For this problem, the set $\X$ is a product of two sets, $\X_c$ the context set and $\X_a$ the action set, and $f:\X \to \mathbb{R}$ 
    is  the mean reward observed for a context-action pair. As before, we will assume that the unknown function $f$ is a sample from a Gaussian process $GP(0,K)$ now indexed by the product set $\X = \X_c\times \X_a$. 
    In each  round $\tau$, the agent receives a \emph{context} $x_{\tau}^c \in \X_c$, and must select an action $x_{\tau}^a \in \X_a$ corresponding to that context and observe the reward $y_{\tau} = f(x_{\tau}) + \eta_{\tau}$ where $x_{\tau} = (x_{\tau}^c,x_{\tau}^a) \in \X$. 
%    For every context $x^c$ there is a context-specific optimal 
%     action defined as 
%    \begin{equation}
%     \label{eq:context_specific_action}
%     x^{a*}(x^c) = \argmax_{x^a \in \X_a}f(x^c,x^a)
%    \end{equation}
The goal of the agent is to design a strategy of selecting actions to minimize the \emph{contextual cumulative regret}:
\begin{equation}
 \label{eq:contextual_regret}
 \mathcal{R}_n^c \coloneqq \sum_{\tau=1}^n \Delta^c(x_{\tau}),
\end{equation}
where we have
\begin{equation}
 \Delta^c(x_{\tau}) \coloneqq \sup_{x^a \in \X_a}f(x_{\tau}^c,x^a) - f(x_{\tau}^c,x_{\tau}^a).
\end{equation}

\subsubsection{Tree based algorithm for Contextual GP bandits}
\label{subsubsec:contextual_gp_algo}
We again make the assumption that the space $\X$ admits a tree of partitions satisfying the properties described in Definition~\ref{def:well_behaved}. To simplify the description of the algorithm, we will assume that the metric space admits a binary tree of partitions (i.e., $N=2$). 
We show  that with a small modification  to  the point selection rule and the cell expansion strategy, we can easily adapt Algorithm-\ref{alg:gp_tree} to the problem of contextual GP bandits. 

We need to introduce a couple of definitions in order to describe the algorithm. 
We call a cell $\X_{h,i}$ active with respect to a context $x^c \in \X_c$, if there exists an action $x^a \in \X_a$ such that $(x^c,x^a) \in \X_{h,i}$. 
Now, given a context $x^c$, for every active cell $\X_{h,i}$ (corresponding to a point $x_{h,i} \in \LL_t$) we find a point of the from $(x^c,x_{h,i}^a) \in \X_{h,i}$, and we will refer the collection of these points as a leaf set relevant to the context $x^c$ denoted by $\LL_t^{rel}$. 

Suppose a cell $\X_{h,i}$ with $0<h<h_{\max}$ is expanded by the algorithm at time $t_0$. Then for all $t\geq t_0$, we use $\bar{x}_{h,i}^{(t)}$ to denote the  candidate  point in the cell $\X_{h,i}$ which was chosen by the algorithm at time $t_0$. 
This point has the property that $\beta_n\sigma_{t-1}(\bar{x}_{h,i}^{(t)}) \leq V_h + \beta_ng(v_1\rho^h)$ for all $t\geq t_0$. Clearly, this property is true at time $t=t_0$ (by Line~6 of Algorithm~\ref{alg:contextual_tree}). Furthermore, since the posterior variance at a point cannot increase as more observations are made, the inequality holds for all $t>t_0$ as well.

For all points in $\LL_t^{rel}$, we define as index as follows:
\begin{equation}
\label{eq:index_contextual_tree}
I_t^c(x_{h,i}) = \min\{ \mu_{t-1}(x_{h,i}) + \beta_n\sigma_{t-1}(x_{h,i}),  \mu_{t-1}(\bar{x}^{(t)}_{h-1,\floor{i/2}}) +\beta_n\sigma_{t-1}(\bar{x}^{(t)}_{h-1,\floor{i/2}}) + V_{h-1} \} + V_h
\end{equation}

The rest of the algorithm proceeds in a manner similar to Algorithm~\ref{alg:gp_tree}. We select a candidate point by maximizing the index $I_t^c$ over the relevant leaf set $\LL_t^{rel}$. Having selected the candidate point, we either evaluate the function or refine the discretization depending on the uncertainty in the function value at the chosen point. 

The steps of the algorithm are shown as a pseudo-code in Algorithm~\ref{alg:contextual_tree}. 
The values of the parameters $h_{\max}$, $\beta_n$ and $(V_h)_{h\geq0}$  used here are the same as those used in the algorithms for GP bandits, with the modification that $n$ now 
represents the total number of context arrivals and $\X = \X_c\times \X_a$.

\begin{algorithm}[H]
		\label{alg:contextual_tree}
		\DontPrintSemicolon
		\SetAlgoLined
		\SetKwInOut{Input}{Input}
        \SetKwInOut{Output}{Output}
        \SetKwInOut{Initialize}{Initialize}
		\Input{$n>0$, ($\X_h)_{h\geq0}$,  $(V_h)_{h\geq 0}$, $h_{\max}$ }
		\BlankLine
		\Initialize{$\LL_0 = \{ x_{0,1} \}$, $t=0$, $\tau=1$, $flag$=TRUE }
        		\BlankLine

		\While{$\tau \leq n$}{
        Observe a context $x_{\tau}^c$ \;
        
        	\While{flag}{
            Obtain $\LL_t^{rel}$\;
            choose $x_{h_t,i_t} =(x_{\tau}^c,x_t^a) \in \argmax_{x_i \in \LL_t^{rel}} I_t^c(x_{h,i})$ \;
            \uIf{ $\beta_n \sigma_{t-1}(x_{h_t,i_t}) \leq V_{h_t} + \beta_ng(v_1\rho^{h_t})$ AND $h_t < h_{\max}$ }{
            $\LL_{t+1} = \LL_t \setminus \{ x_{h_t,i_t}  \}$ \; 
				$\LL_{t+1} = \LL_{t+1} \cup \{ x_{h_t+1,2i_t-1}, x_{h_t+1,2i_t}  \}$ \;				
			}
			\Else{
				play the action $x_t^a$ \;
						observe the reward $y_t = f(x_{h_t,i_t}) + \eta_t$ \;
                update posterior $\mu_t(x)$ and $\sigma_t(x)$ \;
				$flag$ = FALSE \; 

			}
            $t \leftarrow t+1$ \;
            }
		$\tau \leftarrow \tau + 1$ \;
        $flag$ = TRUE \;
		} 
        \BlankLine
		\caption{Tree based Algorithm for Contextual GP bandits}
	\end{algorithm}

\subsubsection{Bounds on contextual regret}
\label{subsubsec:contextual_gp_regret}
For the algorithm for contextual GP bandits described above, we now present high probability bounds on the contextual regret $\mathcal{R}_n^c$:

\begin{theorem}
\label{theorem:contextual_cr}
	Suppose  Algorithm~\ref{alg:contextual_tree} is applied to a contextual GP bandits problem, where the reward function $f$ is a sample from a zero mean GP with covariance function $K \in \mathcal{K}$ and furthermore $K$ is assumed to be isotropic\footnote{ i.e., covariance between two points $x_1$ and $x_2$ satisfies $K(x_1,x_2) = K( l(x_1,x_2))$}
    The product space $\X = \X_c\times \X_a$ is assumed to be \emph{well-behaved} (Definition~\ref{def:well_behaved}) with finite metric dimension $D_1$.   Then after observing $n$ contexts, we have for any $u>0$ with probability at least $1-2e^{-u}$:

\begin{equation}
\mathcal{R}_n^c \leq \tilde{\mathcal{O}}(n^{1-\alpha/(\tilde{D}+2\alpha)}).
\end{equation}
In addition if we further assume that $K(x,x) \leq 1$ for all $x \in \X$, then we can also have an information type bound on the contextual regret:
\begin{equation}
\mathcal{R}_n^c \leq \mathcal{O}(\sqrt{n\gamma_n\log(n)}).
\end{equation}

\end{theorem}

The proof of the above result essentially follows the same arguments used in the proof of Theorem~\ref{theorem:regret_gp_bandits}, and we omit the details here. For deriving the dimension type contextual regret bound, we will require an intermediate lemma analogous to Lemma~\ref{lemma:tree_lemma1}. The derivation of this result differs from Lemma~\ref{lemma:tree_lemma1} in the following two ways:

\begin{itemize}
 \item Unlike Algorithm~\ref{alg:gp_tree}, a single point cannot be evaluated repeatedly in the contextual case as the contexts are not chosen by the algorithm. Thus to get a bound on the term $q_h$ here, we need to upper bound the 
 posterior variance at a point given a certain number of function evaluations at points in a ball $B(x,r,l)$. For this we use the result in the second part of Proposition~\ref{prop:posterior_variance}. 
 
 \item The definition of $\bar{x}_{h,i}^{(t)}$ introduced earlier is crucial in obtaining a bound on the sub-optimality of the chosen action analogous to that in (\ref{eq:delta_1}). 
 Suppose the algorithm selects an action $x_t^c$ which is at level $h_t$ of the tree, in response to a context $x_{\tau}^c$ and let $x_{\tau}^* \coloneqq (x_{\tau}^c,\argsup_{x^a\in \X_a}f(x_{\tau}^c,x^a))$ and  
 $x_{h_t,i_t}= (x_{\tau}^c,x_t^a)$ (note that $\tau$ is the index of the context (i.e.,$1\leq \tau \leq n$) and $t$ is the index of the round (i.e.,$1\leq t\leq n h_{\max}$) in Algorithm~\ref{alg:contextual_tree}). 
We then proceed as follows:
 \begin{align*}
  f(x_{\tau}^*) &\leq \mu_{t-1}( \bar{x}_{h_t-1,\lceil i_t/2 \rceil} ) + \beta_n\sigma_{t-1}(\bar{x}_{h_t-1,\lceil i_t/2 \rceil}) + V_{h_t-1} + V_{h_t}\\
  & \leq f(\bar{x}_{h_t-1,\lceil i_t/2 \rceil}^{(t)}) + 2\beta_n\sigma_{t-1}(\bar{x}_{h_t-1,\lceil i_t/2 \rceil}^{(t)}) + V_{h_t-1} + V_{h_t} \\
  & \leq f(x_{h_t,i_t}) + V_{h_t-1} + 2\beta_n\sigma_{t-1}(\bar{x}_{h_t-1,\lceil i_t/2 \rceil}^{(t)}) + V_{h_t-1} + V_{h_t} \\
  &\stackrel{(a)}{ \leq} f(x_{h_t,i_t}) + 2(V_{h_t-1} + \beta_ng(v_1\rho^{h_t-1}) + 2V_{h_t-1} + V_{h_t} \\
  & \stackrel{(b)}{\leq} f(x_{h_t,i_t}) + (9/2)V_{h_t-1} + 2\beta_n g(v_1\rho^{h_t-1}) \\
  \Rightarrow \Delta^c(x_{\tau}^c,x_t^a) &\leq (9/2)V_{h_t-1} + 2\beta_n g(v_1\rho^{h_t-1})  \coloneqq V_{h_t-1}' = \Oh( g(v_1\rho^{h_t-1})\sqrt{ u + \log n})
 \end{align*}
The inequality $(a)$ above uses the definition of $\bar{x}_{h_t-1,\lceil i_t/2 \rceil}^{(t)} $ and $(b)$ uses the fact that $2V_{h_t} \leq V_{h_t-1}$. 
\end{itemize}
With these results available, the remainder of the proof of Theorem~\ref{theorem:contextual_cr} mirrors the proof of Theorem~\ref{theorem:regret_gp_bandits}.

\begin{remark}
Compared to the CGP-UCB algorithm of \cite{krause2011contextual}, Algorithm~\ref{alg:contextual_tree} again has two  benefits. First, if $\X \subset \mathbb{R}^D$, then the computational cost of running the algorithm 
does not depend on the dimension of the space, unlike the CGP-UCB whose practical implementation cost increases exponentially with $D$. Second, as with Algorithm~\ref{alg:gp_tree}, our theoretical regret bounds are tighter for Mat\'ern kernels when we have $D \geq \nu -1$. 
\end{remark}

\begin{remark}
\cite{krause2011contextual} considered composite covariance functions formed either by taking products $K(x^c,c^a) = K_c(x^c)\times K_a(x^a)$ or by taking sums $K(x^c,x^a) = K_c(x^c) + K_a(x^a)$ of different covariance functions over the context space and the action space. Since our class of covariance functions $\mathcal{K}$ is closed under such operations, if  $K_c$ and $K_a$ lie in $\mathcal{K}$ then their composition will also be in $\mathcal{K}$, and thus our dimension-type bounds on the contextual regret is valid for such composite covariance functions. In addition, for the information type bound we can use \citep[Theorem 2 and Theorem 3]{krause2011contextual} to get the required upper bound on $\gamma_n$. 
\end{remark}
%____________________________________________________________________________________________________________________

\section{Technical Results}
\label{section:technical_results}

       In this section, we present some analytical results about the Gaussian Processes satisfying the assumptions described in Section~\ref{subsubsec:assumptions_on_K}, which were used in the design of our algorithms. 

	We begin  by deriving a high probability bound on the maximum variation of the sample functions of a Gaussian Process within a $d$-ball of radius $b$ around 
	some fixed point $x$. 

	\begin{proposition}
		\label{prop1}
		Suppose $\{f(x); x \in \X\}$ is a separable zero mean Gaussian Process $GP(0,K)$, and let $d$ denote the usual metric on $\X$ induced by the GP. Let $B(x_0,b,d)
		\subset \X$ be a $d$-ball of radius $b>0$. Then we have for any $u >0$:
		\begin{equation}
		\label{eq:chain1}
		Pr\big(\sup_{x \in B(x_0,b,d)}|f(x) - f(x_0)| > w_b \big) \leq e^{-u},
		\end{equation}
		with $w_b   \leq 4b\big( \sqrt{C_2 + 2u + 2D_1' \log(1/b)} + C_3 )$. Here $C_2$ and $C_3$ are positive constants and $D_1'$ is the metric dimension of $B(x_0,b,d)$ 
		with respect to $d$. 
		
	\end{proposition}
	
	The details of the proof of this statement is given in Appendix ~\ref{proof_prop1}. The proof uses the classical chaining technique for bounding the suprema of Gaussian Processes, and follows the same line of arguments used in some existing results in literature such as
    \citep[Theorem~3.3]{contal2016thesis} and \citep[Theorem~5.24]{van2014probability}.
% 	Chaining is a useful tool for studying the suprema of Gaussian Process, and hence it has found wide use in non-parametric
% 	statistics literature (See for example \cite{gaillard2015chaining},\cite{cesa2017algorithmic}\cite{contal2016stochastic}).

The previous result gives us a bound on the variation of the samples of a given Gaussian process  within a given $d$-ball of radius $b$. 
	Using this and the union bound, we can easily extend this to a sequence of discretizations of $\X$:
	
	\begin{corollary}
		\label{cor:chaining}
		Suppose $\{f(x); x \in \X \}	$ is a zero mean Gaussian Process which induces the  metric $d$ on $\X$.
		Let $(\X_k)_{k\geq 0}$ be a sequence of finite 
		subsets of $\X$, and to every point in $\X_k$ we associate a radius $b_k$ with respect to the metric $d$. Then we have $Pr(\Omega_u) \geq 1-e^{-u}$, where the event 
		$\Omega_u$ is defined as
		\begin{equation}
		\Omega_u = \{ \forall n \geq 0, \forall x \in \X_k: \sup_{y \in B(x,b_k,d)}|f(y)-f(x)| \leq w_k  \},
		\end{equation}
		with the value of $w_k$ given by:
		\begin{equation}
		w_k  \leq 4b_k\big( \sqrt{C_4 + 2u + 2\log(|\X_k|/(b_k^{D_1'}))} + C_3 ),
		\end{equation}
		where $C_4 = C_2 + 2\log(n^2\pi^2/6)$, and $D_1'$ is the metric dimension of $(\X,d)$.
	\end{corollary}

	\begin{proof}
		The result is obtained by replacing $u_k \leftarrow u_k + \log(n^2\pi^2/6) + \log(|\X_k|)$ in the proof of Proposition~\ref{prop1} and then taking two union bounds, one over points in $\X_k$ for a fixed $n$ and the other over all values of $n \in \mathbb{N}$. 
	\end{proof}

	Specializing this result to the class of Gaussian Processes with covariance functions $K \in \K$, we can obtain bounds on the variation of the GP samples in $l$-balls. 
	\begin{corollary}
	 \label{cor:chaining2}
	 Suppose $\{f(x); x \in \X \}$ is a Gaussian Process with its covariance function $K \in \K$, and let $l$ be a metric defined on $\X$.
	  Then for $(\X_k)_{k\geq 0}$ subsets of $\X$, and $(r_k)_{k\geq 0}$ the associated radius values, we have for any $u>0$:
	 \[ P(\Omega_{u1}) \geq 1 - e^{-u},\] where the event $\Omega_{u1}$ is defined as 
	 \begin{equation}
		\Omega_{u1} = \{ \forall n \geq 0, \forall x \in \X_k: \sup_{y \in B(x,r_k,l)}|f(y)-f(x)| \leq w(r_k)  \},
		\end{equation}
		with the value of $w(r_k)$ given by:
		\begin{equation}
		\label{eq:chaining4}
		w(r_k)  \leq 4g(r_k)\big( \sqrt{C_4 + 2u + 2\log(|\X_k|/( g(r_k)^{D_1}  )} + C_3 ).
		\end{equation}

	\end{corollary}

	This result gives us control over the variation of the Gaussian process samples in balls centered at points in $(\X_k)_{k\geq 0}$.
	
	Now suppose we want to obtain high probability
	bounds on the variation of the GP samples in $l$-balls of radius $(r_k)_{k\geq 0}$ for all points $x \in \X$ and not just those in $(\X_k)_{k\geq 0}$. Our next result shows that
	we can obtain this by a small modification of the previous result.

	\begin{proposition}
	 \label{prop:chaining3}
	For a given sequence $(r_k)_{k\geq 0}$, we have for any $u >0$,  $Pr(\Omega_{u2}) \geq 1- e^{-u}$, where the event $\Omega_{u2}$ is defined as
	\begin{equation}
	 \label{eq:lip_zoom_event}
	 \Omega_{u2}= \{ \forall k\geq 1, \forall x \in \X: \sup_{y \in B(x,r_k,l)}|f(x)-f(y)| \leq \tilde{w}_k  \}, 
	\end{equation}
	and $\tilde{w}_k = 2w(R_k)$, where $w(R_k)$ is as defined in  ~\ref{eq:chaining4} by selecting $\X_k$ to be an $\epsilon_k$ cover (for any $\epsilon_k>0$) of $\X$, and choosing $R_k$ satisfying
	$R_k \geq r_k+\epsilon_k$. 
	\end{proposition}
This result is crucial in the design of Algorithm~\ref{alg:zoom} as the covering oracle can return an arbitrary point in the uncovered region of the search space $\X$, and thus we need to  bound the variation of $f$ in ball centered at \emph{any} point $x\in \X$ with radius $r_k$ for $k\geq 0$.

	\begin{remark}	
	 The result follows by application of Corollary-\ref{cor:chaining2} for the given choice of $R_k$ and $\epsilon_k$. 
	 
	 However, the idea behind this result can be better  understood through Figure.\ref{fig:figure1}. 
	 Let us consider a fixed radius $r_k$. We want a bound $\tilde{w}_k$ such that for all $ x \in \X$ we know that with high probability the variation of a
	Gaussian process sample within the ball $B(x,r_k)$ is no more than $\tilde{w}_k$. Since the set $\X$ in general can be uncountable, we cannot directly use union bound to get this result.
	However, we can get a bound in the following way: For some $\epsilon_k >0$, consider an $\epsilon_k$ covering of $\X$, denoted by $\X_{\epsilon_k}$. Now for every point 
	$z \in \X_{\epsilon_k}$, we associate a ball $B(x,R_k,l)$ with $R_k \geq r_k + \epsilon_k$ and compute the corresponding variation ($w(R_k)$) within this ball for all
	$x \in \X_{\epsilon_k}$ by using Corollary~\ref{cor:chaining2}. By definition of  $\X_{\epsilon_k}$, for all $x\in \X$ there exists a $z_x$ within 
	$\epsilon_k$ distance of $x$, and by the choice of radius $R_k$, we know that $B(x,r_k,l) \subset B(z_x,R_k,l)$. Now, by the triangle inequality, we have for all $y \in B(x,r_k,l)$, 
	$|f(x)-f(y)| \leq |f(x)-f(z_x)|+|f(y)-f(z_x)|$, which gives us the required bound  $\tilde{w}_k \leq 2w(R_k)$. 
	 
	\end{remark}

	\begin{figure}
		\centering
		\includegraphics[width=2in, height = 2in]{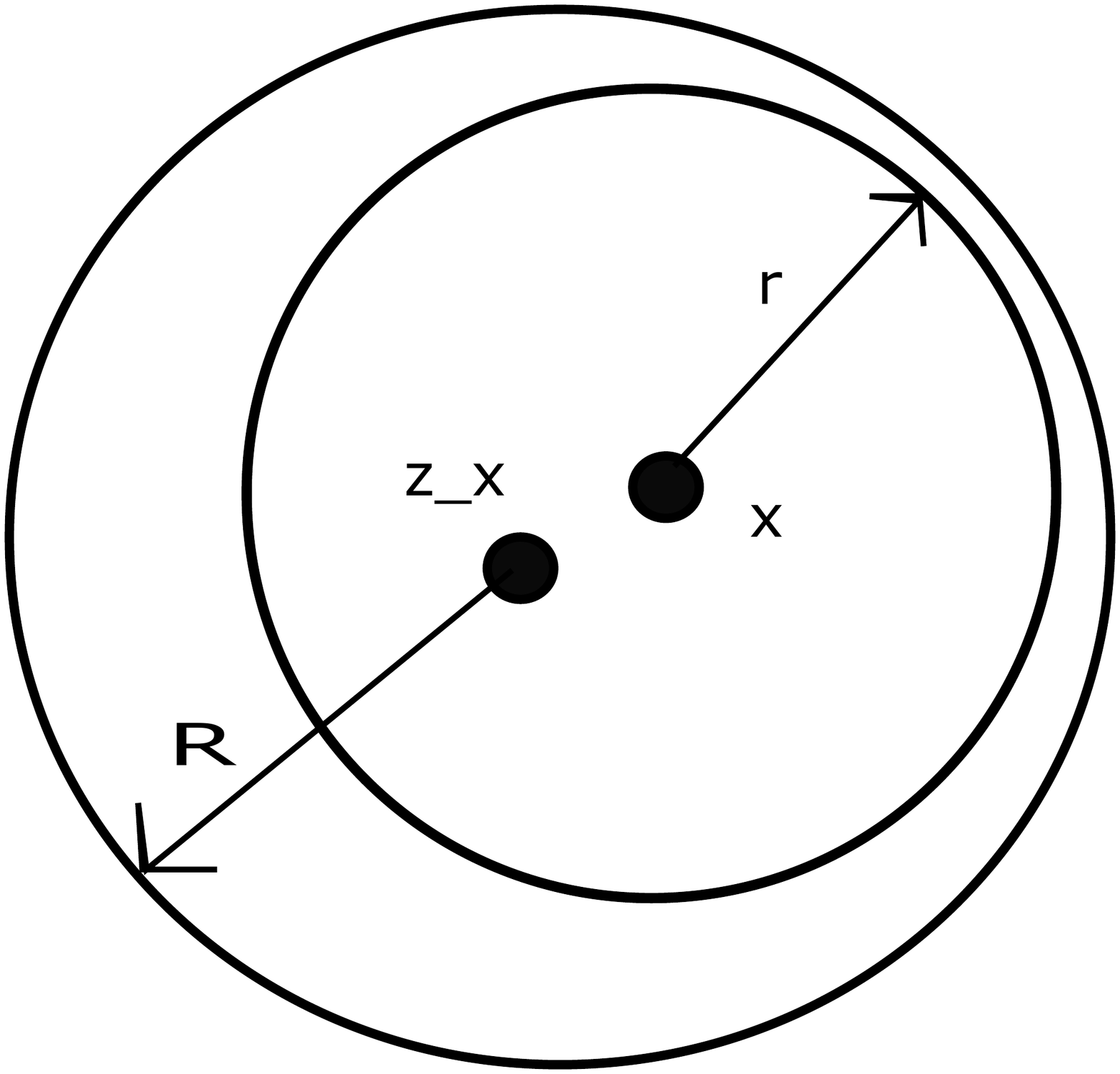}
		\caption{If $\X_{\epsilon}$ is an $\epsilon$ cover of $\X$, then for any $x \in \X$ there exists a $z_x \in \X_{\epsilon}$ within $\epsilon$ distance of $x$. A ball 
		of radius $R\geq r+\epsilon$ will contain the ball $B(x,r)$ and so twice the variation of $f$ in $B(z_x,R)$ (denoted by $w(R)$) is an upper bound on the variation of
		$f$ in $B(x,r)$  }
		\label{fig:figure1}
	\end{figure}

	Finally, we present a result about the posterior variance at a point $x$ at which we have multiple noisy observations.

    	\begin{proposition}
        \label{prop:posterior_variance}
		Suppose the unknown function $f$ is a sample from $GP(0,K)$ with $K \in \mathcal{K}$. 
        \begin{itemize}
        \item 
        If a point $x$ has been evaluated $n_t(x)$ times before time $t$ according to the observation model 
		$y(x) = f(x) + \eta$ where the noise term $\eta$ is distributed according to $N(0,\sigma^2)$. 
		 Then we have 
		\begin{equation}
		\sigma_t(x) \leq \frac{\sigma}{\sqrt{n_t(x)}},
		\end{equation}
		where $\sigma_t(x)$ is the posterior variance at the point $x$ after $t$ observations. 
        
        \item Suppose, we make the further assumption that the covariance function $K$ is isotropic, i.e., $K(x,y) = K(r)$ where $r =  l(x,y)$. Now, if $n_t(x,r)$ denotes the number of times a point from the ball $B(x,r,l)$ has been evaluated up to time $t$, then we have
        \begin{equation}
        \sigma_t(x) \leq \frac{\sigma}{\sqrt{n_t(x,r)}} + d(r) \leq \frac{\sigma}{\sqrt{n_t(x,r)}} + g(r).
        \end{equation}
        \end{itemize}
	\end{proposition}
    This result allows us to estimate the number of  evaluations  required to bring the uncertainty about the function value at a point below a certain threshold. The first part of the above result is used in the analysis of the two 
    algorithms proposed for GP bandits (Algorithm~\ref{alg:gp_tree} and Algorithm~\ref{alg:zoom}), while the second part is used in the analysis of Algorithm~\ref{alg:contextual_tree} for Contextual GP bandits \citep{krause2011contextual}.

%________________________________________________________________________________________________________________________________________________________________________________________________________________________________________________________

\section{Conclusion}
\label{section:conclusion}

In this paper, we considered the problem of optimizing an unknown function under noisy bandit feedback, and presented an algorithm which adaptively discretizes the search space using a hierarchical tree of partitions. We then obtained high probability bounds on the cumulative and simple regret for our algorithm.
Because of adaptive refinement of the search space, our algorithms can be computationally much cheaper than the existing approaches using uniform discretizations. Furthermore, we  also identified sufficient conditions under which the regret bounds of our algorithms improve upon the existing theoretical results. 

Finally, we note that the tools described in Section~\ref{section:technical_results}, along with some stronger bounds on suprema of GPs such as those presented in \citep{contal2016stochastic,van2015chaining} may be useful for designing adaptive algorithms for some other settings, such as time varying GP bandits problem \citep{bogunovic2016time}.

\newpage

\appendix
\label{section:appendix}

\section{Details of Toy examples in Section~\ref{subsec:toy_examples}}
\label{appendix:details_toy_example}

\subsection{Example~1}
\label{appendix:toy_example1}
First we note that the covariance function of the Gaussian Process is uniformly upper bounded by $a_1^2$ which implies that the information type regret bound is valid for it \citep{srinivas2012information}. Before obtaining the lower bound on $\gamma_n$, let us select the parameters $(a_i)_{i\geq 1}$ in the following way for a fixed $\delta>0$: 
    \begin{align*}
    a_1 &= \frac{1}{(\Phi)^{-1}((1+\delta)/2)} \\
    a_i &= \frac{1}{2\sqrt{ 2\log(\frac{\pi^2i^2}{6\delta})  }}
    \end{align*} where $\Phi(.)$ is the cdf of Standard Normal random variable.

Now, using \citep[Lemma 5.3]{srinivas2012information}, we have 
\begin{align*}
\gamma_n &\geq \sum_{i=1}^n\log\big(1 + \frac{a_i^2}{\sigma^2}   \big) 
\geq n\log\big(1+\frac{a_n^2}{\sigma^2} \big) \\
& \stackrel{(a)}{\geq} n\frac{a_n^2}{a_n^2+\sigma^2} 
= \frac{n}{1 + 8\sigma^2\log\big( \frac{\pi^2n^2}{3\delta} \big)}
\end{align*}
where $(a)$ follows from the inequality $\log(1+x) \geq \frac{x}{1+x}$ for $x \geq 0$. From the above, we get the following bound:
\[
\gamma_n \geq n \max \bigg( \frac{1}{2}, \frac{1}{16\sigma^2 \log\big( \frac{\pi^2n^2}{3\delta} \big) } \bigg)
\]

This implies that for all $\sigma^2>0$, the information type regret bound for this Gaussian Process increases linearly with $n$.

Now we show that for the given  choice of parameters for this Gaussian Process, the global maximizer of the sample function $f$ can be found from just one evaluation with high probability. Let us define the following events: $E_1 =\{ |a_1X_1|\geq 1 \}$, $E_2 = \{ \forall i \geq 2: |a_iX_i| \leq 1/2 \}$ and $E_3 = \{ |\eta_1|\leq 1/2 \}$ where $\eta_1$ is the observation noise at time $t=1$. Then,   we have $Pr(E_1\cap E_2 \cap E_3) \geq 1-3\delta$, and it is easy to see that the global maximum of the function $f$ under the event $E_1\cap E_2 \cap E_3$ will lie either at $x=1/2$ or $x=5/6$. 
Since, by construction we have $f(1/2) = -f(5/6)$, a single evaluation of the function at either of these two points is sufficient to find the global maximum, and hence the regret $\mathcal{R}_n \leq \Oh(1)$.

\subsection{Example~2}
\label{appendix:toy_example2}
We observe that the covariance function of the Gaussian Process is upper bounded by $\sum_{i=1}^{\infty}a_i^2$ which for our choice of parameters $a_i$ will be finite. 
If we make the extra  assumption that the noise variance is smaller than $a_n^2$, we get that $\gamma_n \geq n\log(2)$ which implies that the information type regret bound increases linearly with $n$.

Now, for a fixed $\delta>0$,  let us define the event \[ E_4 = \{ |X_i| \leq \sqrt{2\log\big( \frac{\pi^2i^2}{3\delta} \big)}  \hspace{1em} \text{for all } i\geq 1 \}. \] By using the tail bounds for Gaussian random variables and the union bound, we get that $Pr(E_4) \geq 1-\delta$. 
We now set the parameters as follows: 
\begin{equation*}
\begin{aligned}
a_i &= \frac{1}{i^2 \sqrt{2\log\big( \frac{\pi^2i^2}{3\delta} \big)}} \hspace{1em} \text{for all } i \geq 1\\
\sigma &= a_n/\sqrt{2}
\end{aligned}
\end{equation*}
Next, suppose $(\eta_t)_{t\geq1}$ denote the  $i.i.d.$ $N(0,\sigma^2)$ noise random variables. We define the following event which also occurs with probability at least $1-\delta$: 
\[
E_5 = \{ |\eta_t| \leq 1/(\sqrt{2}t^2) \hspace{1em} \text{for all }  t \geq 1 \}
\]

Now, we need to show that there exists a strategy which will ensure with high probability that the cumulative regret is upper bounded by $\mathcal{O}(\log(n))$. Assuming that the events $E_4$ and $E_5$ hold (which happens with probability at least $1-2\delta$), we proceed as follows:

\begin{itemize}
\item We first note that we can construct a ternary tree of intervals ($\{ \mathcal{I}_{j,k}: j\geq 0, \text{ and } 1\leq k\leq 3^j \})$ which form an increasing sequence of partition of the input space $\X = [0,1]$. The root of the tree is the entire unit interval $\mathcal{I}_{0,1} = [0,1]$ while the nodes at level $1$ are obtained by partitioning $\mathcal{I}_{0,1}$ into three equal intervals $\mathcal{I}_{1,1}=[0,1/3)$, $\mathcal{I}_{1,2}=[1/3,2/3)$ and $\mathcal{I}_{1,3}=[2/3,1]$. This process is repeated indefinitely to get an infinite ternary tree. 

\item Because of the definition of the Gaussian Process, the function value in the interval $\mathcal{I}_{1,1}$ is $a_1X_1\varphi(3x) + f_2(3x)$ and in the interval $\mathcal{I}_{1,3}$ is $-a_1X_1\varphi(3x-2) + f_2(3(x-2/3))$, we note that $x^*$ must lie either in $\mathcal{I}_{1,1}$ or $\mathcal{I}_{1,3}$. To decide which one, we need to know the sign of $X_1$ for which we observe the function at the mid point of the interval $\mathcal{I}_{1,2}$. If the observed value is positive, we can conclude that $x^*$ must lie in $\mathcal{I}_{1,1}$. Otherwise, $x^*$ lies in $\mathcal{I}_{1,3}$. Thus our region of uncertainty shrinks from $\mathcal{I}_{0,1}$ to $\mathcal{I}_{1,1}$ or $\mathcal{I}_{1,3}$. 

\item For $t>1$, we proceed similarly by evaluating the function at a point $x_t$ in the middle sub-interval of the current region of uncertainty. Based on the observed value, we can infer the sign of $a_tX_t$ which allows us the pick the next subinterval. Thus at any time $t$, the suboptimality of the evaluated point is upper bounded by \[ f(x^*)-f(x_t) \leq \sum_{i\geq t} |a_iX_i| + |\eta_i| \leq  \sum_{i\geq t}\frac{2}{i^2} \leq \frac{2}{t}\]
where the second inequality follows from the definition of event $E_4$ and the choice of $(a_i)_{i\geq 1}$. 

\item Finally, summing up all such terms gives us the required bound on the cumulative regret
\[
\mathcal{R}_n \leq \sum_{t=1}^n \frac{2}{t} \leq 2\log(n)
\]

\end{itemize}

\section{Deferred proofs from Section~\ref{section:technical_results}}
\label{appendix:proof_technical_results}
\subsection{Proof of Proposition~\ref{prop1}}
		\label{proof_prop1}
		
%         \textbf{Proposition 1.}
%        	Suppose $\{f(x); x \in \X\}$ is a separable zero mean Gaussian Process $GP(0,K)$, and let $d$ denote the usual metric on $\X$ induced by the GP. Let $B(x_0,r,d) \subset \X$ be a $d$-ball. Then we have for any $u >0$:
% 		\begin{equation}
% 		Pr\big(\sup_{x \in B(x_0,r,d)}|f(x) - f(x_0)| > w_r \big) \leq e^{-u}
% 		\end{equation}
% 		with $w_r   \leq 4r\big( \sqrt{C_2 + 2u + 2D_1' \log(1/r)} + C_3 )$. Here $C_2$ and $C_3$ are positive constants and $D_1'$ is the metric dimension of $B(x_0,r,d)$ with respect to $d$. 

		\begin{proof}
			Let $T = B(x_0,r,d)$ and let us assume we have a sequence of increasingly fine discretizations $(T_n)_{n\geq 0}$ of $T$ with $T_0=\{x_0\}$, and let $\pi_n:T\rightarrow T_n$ represent the projection operator onto $T_n$, i.e., $\pi_n(x) = \argmin_{y \in T_n} d(x,y)$. Then we have the following:
			\begin{equation*}
			|f(x)-f(x_0)| = |\sum_{n\geq 1} f(\pi_n(x)) - f(\pi_{n-1}(x))| \leq \sum_{n\geq 1}|f(\pi_n(x)) - f(\pi_{n-1}(x))|
			\end{equation*}
			
			Now we use the concentration property of Gaussian Process (~\ref{eq:gaussian_tail}) and union bounds, to get:
			\begin{align*}
			Pr\big(  |f(\pi_n(x)) - f( \pi_{n-1}(x))| > \sqrt{u_n}d(\pi_n(x), \pi_{n-1}(x))     \big) & \leq 2\exp(-u_n/2) \\
			\Rightarrow Pr \big(	\exists x \in T: |f(\pi_n(x)) - f( \pi_{n-1}(x))| > \sqrt{u_n}d(\pi_n(x), \pi_{n-1}(x)) 	  	\big) &\leq 2|T_n||T_{n-1}|\exp(-u_n/2)\\
			\Rightarrow Pr \big(\exists n \in \N, 	\exists x \in T: |f(\pi_n(x)) - f( \pi_{n-1}(x))| > \sqrt{u_n}d(\pi_n(x), \pi_{n-1}(x)) 	  	\big) &\leq \sum_{n\geq 1}2|T_n||T_{n-1}|\exp(-u_n/2) \\
			& \coloneqq P_e
			\end{align*}
			
			Let us define the event $E_1 = \{ \exists n \in \N, \exists x \in T: |f(\pi_n(x)) - f(\pi_{n-1}(x))| > \sqrt{u_n}d(\pi_n(x), \pi_{n-1}(x))  \}$. Then under the event $E_1^c$, we know that for all $x$ and $n$, we have $|f(\pi_n(x)) - f(\pi_{n-1}(x))| \leq \sqrt{u_n}d(\pi_n(x),\pi_{n-1}(x))$, which means that 
			\begin{align*}
			\sup_{x \in T}|f(x)-f(x_0)| &\leq \sup_{x\in T}\sum_{n\geq 1}|f(\pi_n(x))-f(\pi_{n-1}(x))| \\
			& \leq \sup_{x\in T}\sum_{n \geq 1}\sqrt{u_n}d(\pi_n(x),\pi_{n-1}(x)) 
			\end{align*}
			Now, let us choose $T_n$ to be the $\epsilon_n = r2^{-n}$ covering of $T$ with respect to the metric $d$. Assuming that $T$ has a finite metric dimension $D_1'$, we have $|T_n| \leq C_12^{nD_1'}/r^{D_1'}$. 
			Now in order to keep $P_e$ below $e^{-u}$ for some $u>0$,  we set $u_n = 2(u + v_n)$ with $v_n$ to be defined later. This gives us 
			\[
			P_e \leq \sum_{n\geq 1}2(C_1^22^{(2n-1)D_1'}/r^{2D_1'})e^{-u_n/2} \leq \sum_{n \geq 1}2C_1^2\frac{2^{2nD_1'}}{r^{2D_1'}}e^{-u_n/2}
			\]
			Now by choosing \[ v_n =\log\bigg(2C_1^2\frac{2^{2nD_1'}}{r^{2D_1'}}\bigg) + 
			\log(n^2\pi^2/6 ) \]
			we get the required bound on $P_e$. Now it remains to get the upper bound on $w_r$ for this choice of $u_n$. 
			We use the fact that $d(\pi_n(x),\pi_{n-1}(x)) \leq d(\pi_n(x),x) + d(\pi_{n-1}(x),x) \leq 2r2^{-(n-1)}$ to get 
			\[
			w_r \leq 2r\sum_{n\geq 1}2^{-(n-1)} \sqrt{2u +2D_1'\log(1/r)+ 2\log(n^2) + 2nD_1'\log(2) + 2\log(2C_1^2\pi^2/6) }
			\]
			Finally, replacing $2\log(2C_1^2\pi^2/6)$ with $C_2$, and writing $\sum_{n\geq 1}2^{-(n-1)}\sqrt{\log n} = \alpha_1$ and 
			$\sum_{n\geq 1}2^{-(n-1)}\sqrt{n} = \alpha_2$, we get 
			
			\begin{equation}
			w_r \leq 4r\big( \sqrt{C_2 + 2u + 2D_1' \log(1/r)} + C_3 )
			\end{equation}
			where $C_3 = \alpha_1 + \alpha_2\sqrt{D_1'\log2} $. 
		\end{proof}
		
\subsection{Proof of Proposition~\ref{prop:posterior_variance}}
\begin{proof}
    Let $\bar{y}_{1:t-1}$ denote all the observations before time $t$, and $\bar{y}_x$ be the vector of observations at $x$. Also, let $\bar{y}_{x^c}$ be the vector of observations at points other than $x$. Then by the non-negativity of mutual information we have

		\begin{align*}
		&I(f(x);\bar{y}_{x^c}|\bar{y}_{x}) \geq 0 \\
		\Rightarrow& h(f(x)|\bar{y}_x) - h(f(x)|\bar{y}_{x},\bar{y}_{x^c}) \geq 0\\
		\Rightarrow& \log\bigg(\frac{1}{\sqrt{\frac{n_t(x)}{\sigma^2} + \frac{1}{K(x,x)}}}\bigg) -\log(\sigma_{t}(x) \stackrel{(a)}{\geq} 0\\
		\Rightarrow& \frac{\sigma}{\sqrt{n_t(x)}} \geq \sigma_{t}(x)	
		\end{align*}
		where $h(X)$ is the differential entropy of $X$ and $I(X;Y)$ denotes the mutual information between random variables $X$ and $Y$. For  inequality $(a)$, we used the
		formula for the differential entropy of a Gaussian random variable.
        
        For the second part, let us define $S_x = \{x_1,x_2,\ldots,x_{n_t(x,r)}\}$ as the set of points  in $B(x,r,l)$ which have been evaluated up to time $t$. Further introducing the vector $K_x = [K(x,x_1),K(x,x_2), \ldots , K(x,x_{n_t(x,r)}]^{tr}$ where $tr$ denote the transpose operation,  and the matrix $K_{xx} = [K(x_i,x_j)]_{(x_i,x_j) \in S_x\times S_x}$, we have by the formula for the posterior variance at $x$ \citep[(2.26)]{rasmussen2006gaussian}:
        \begin{align*}
        \sigma_t^2(x) &\leq K(0) - K_x^T(K_{xx} + \sigma^2I)^{-1}K_x 
        \end{align*}
        Now, based on the assumption that $K$ is isotropic, we can make the following two observations, 
        \begin{align*}
        (K_{xx}+\sigma^2I) &\preccurlyeq (K(0)\ind \ind^T + \sigma^2I) \\
        K(r)\ind \preccurlyeq K_x
        \end{align*}
        which gives us
        \begin{equation}
        -K_x^T(K_{xx} + \sigma^2I)^{-1}K_x \leq K(r)^2\ind^T(K(0)\ind \ind^T + \sigma^2I)^{-1}\ind
        \end{equation}
        Now, using the Woodbury matrix inversion identity, and some simplification, we get:
        \begin{align*}
        \sigma_t^2(x) &\leq \frac{K(0)\sigma^2 + n_t(x,r)(K(0)^2 - K(r)^2}{\sigma^2 + n_t(x,r)K(0)} \\
        \Rightarrow \sigma_t^2(x) &\leq \frac{\sigma^2}{n_t(x,r)} + 2(K(0) - K(r))\\
        \Rightarrow \sigma_t(x) &\stackrel{(a)}{\leq} \frac{\sigma}{\sqrt{n_t(x,r)}} + d(r) \\
        \Rightarrow \sigma_t(x) &\stackrel{(b)}{\leq} \frac{\sigma}{\sqrt{n_t(x,r)}} + g(r)
        \end{align*}
        where $(a)$ uses the inequality $\sqrt{z_1 + z_2} \leq \sqrt{z_1} + \sqrt{z_2}$ for $z_1,z_2\geq 0$ and $(b)$ follows from the fact that $K \in \K$. 
	\end{proof}

\section{Proof of Theorem~\ref{theorem:regret_gp_bandits}}
\label{appendix:gp_tree_regret}
For the entirety of this proof, we will assume that the events $\Omega_{u5}$ and $\Omega_{u6}$ hold, which is true with probability at least $1-2e^{-u}$. 
Let $(\tau_j)_{j\geq 1}^n$ denote the rounds in which the function evaluations were performed, and let $Q_n = \{ x_{h_{\tau_j},i_{\tau_j}}| 1\leq j\leq n \}$ denote the multiset of points evaluated by the algorithm. 

\subsection{Information-type bound on $\mathcal{R}_n$}

To obtain the information-type cumulative regret bound, we divide the set $Q_n$ into $Q_{n1}$ and $Q_{n2}$, where 
\[
 Q_{n1} = \{ x_{h,i} \in Q_n | h < h_{\max} \}
\]
and $Q_{n2} = Q_n \setminus Q_{n1}$. 

From Lemma~\ref{lemma:tree_lemma1} , we know that for all $x_{h,i} \in Q_{n2}$, we have $\Delta(x_{h,i}) \leq (2N+1)V_h$, and assuming $n$ is large enough so that $h_{\max} \geq h_0 \coloneqq \frac{\log(v_2/\delta_K)}{\log(1/\rho)}$, we can
upper bound the contribution of the terms in $Q_{n2}$ to the cumulative regret (denoted by $\mathcal{R}_{n2}$) as follows:
\begin{align*}
 \mathcal{R}_{n2} & \coloneqq \sum_{x_{h,i} \in Q_{n2}} f(x^*) - f(x_{h,i}) \\
 & \leq (2N+1)V_{h_{\max}}|Q_{n2}| \\
 & \leq (2N+1)V_{h_{\max}}n \\
 & \leq \Oh( \rho^{h_{\max}\alpha} \sqrt{ h_{\max}} n )
\end{align*}
where the last inequality relies on the assumption  that $h_{\max} \geq h_0$ and the properties of  the covariance functions in the class $\K$. 
Now, using the fact that $h_{\max} \geq \frac{(1/2)\log(n)}{\alpha \log(1/\rho)}$ we get that $\mathcal{R}_{n2} \leq \Oh( \sqrt{n \log(n)})$. 

Now, for the terms $x_{h_{\tau},i_{\tau}}$ in $Q_{n1}$ we observe from Lemma~\ref{lemma:tree_lemma1} that $f(x^*) - f(x_{h_{\tau},i_{\tau}}) \leq 3\beta_n\sigma_{\tau-1}(x_{h_{\tau},i_{\tau}})$. If $|Q_{n1}| = n_1$, then by using 
\citep[Lemma~5.3 and Lemma~5.4]{srinivas2012information}, and the assumption that $K(x,x) \leq 1$ for all $x \in \X$, we get:
\[
 \mathcal{R}_{n1} \leq \Oh( \sqrt{n_1\gamma_{n_1}\log(n_1)}) \leq \Oh( \sqrt{n\gamma_n \log(n)})
\]

On adding the two terms, we get the required information type bound $\mathcal{R}_n \leq \Oh( \sqrt{n \gamma_n \log(n)})$

\subsection{Dimension-type regret bounds}
We first obtain the dimension-type bound on the cumulative regret. 
Recall that the algorithm only selects points for evaluation from the sets of the form $\X_{(2N+1)V_h} = \{ x_{h,i} \in \X: f(x^*) - f(x_{h,i}) \leq (2N+1)V_h\}$, and furthermore, by the assumption on the metric space that 
any two points in $\X_h$ are separated by at least $2v_2\rho^h$. These two facts imply that $|\X_{(2N+1)V_h} \cap \X_h| \leq M(\X_{(2N+1)V_h}, 2v_2\rho^h, l)$. 

We first consider the contribution of the terms $x_{h,i}$ for which $h < h_0$: 
\begin{align*}
\mathcal{R}_1 &= \sum_{x_{h,i} \in Q_n: h<h_0} f(x^*) - f(x_{h,i})  \\
& \leq \sum_{h=0}^{h_0-1} | \X_{(2N+1)V_h} \cap \X_h| q_h \\
& \stackrel{(a)}{\leq} \Oh \bigg( \sum_{h=0}^{h_0-1}  \rho^{-hD_1}q_h  \bigg)\\
& \stackrel{(b)}{\leq} \Oh \bigg( \sum_{h=0}^{h_0-1} \rho^{-hD_1}\frac{\sigma^2\beta_n^2}{g(v_2\rho^{h_0})} \bigg) \\
& = \Oh(\log(n))
\end{align*}
where $(a)$ relies on the fact that $\X$ has a finite metric dimension $D_1$ and $(b)$ uses Lemma~\ref{lemma:tree_lemma1} and the fact that $g$ is a non-decreasing function.

Now, we fix an $H$ such that $h_0 \leq H \leq h_{\max}$. We then have the following: 
\begin{align*}
 \mathcal{R}_2 &= \sum_{x_{h,i} \in Q_n: h_0\leq h\leq H} f(x^*) - f(x_{h,i}) \\ 
  & \leq \sum_{h=h_0}^H |\X_{(2N+1)V_h} \cap \X_h|(2N+1)V_h q_h \\
 & \stackrel{(c)}{\leq} \Oh\bigg( \sum_{h=h_0}^H \beta_n^2 \rho^{-h(\tilde{D} + \alpha)}\sqrt{h}     \bigg) \\
 & = \Oh \bigg( \rho^{H(\alpha+\tilde{D})}\sqrt{H}\log(n) \bigg)
\end{align*}
In the inequality $(c)$ above, we use the fact that for $h\geq h_0$, we have $|\X_{(2N+1)V_h}\cap \X_h| = \Oh(\rho^{h\tilde{D}})$ by using the definition of $\tilde{D}$, $V_h$ is  $\Oh( \rho^{h\alpha}\sqrt{h})$ 
and $q_h =  \Oh( \frac{\beta_n^2}{\rho^{2h\alpha}})$ by the assumptions on the covariance function. 

Finally, the contribution of the remaining points in $Q_n$ can be trivially upper bounded as:
\begin{align*}
 \mathcal{R}_3 &= \sum_{x_{h,i} \in Q_n: h\geq H} f(x^*) - f(x_{h,i})  \leq nV_{H}\\
 & \leq \Oh( n\rho^{H\alpha}\sqrt{H})
 \end{align*}

 Now, if we select $H= \frac{ \log(n)}{\log(1/\rho)\alpha(\tilde{D}+2\alpha)} <h_{\max}$, we get 
 \[
  \mathcal{R}_n \leq \Oh \bigg( \log(n)^{3/2}n^{1-\frac{1}{\tilde{D}+2\alpha}} \bigg)
 \]
as required. 

To obtain the bound on the simple regret, we introduce the terms $Q_h = Q_n \cap \X_h$ for $h\geq0$. for any $H>0$, we have the following:
\begin{align*}
 \sum_{h=0}^H |Q_h| & = \sum_{h=0}^{h_0-1}|Q_h| +\sum_{h=h_0}^H|Q_h| \leq n \\
 & \leq \Oh(1) + \Oh\bigg( \sum_{h=h_0}^H \rho^{-h(\tilde{D}+2\alpha)}\beta_n^2   \bigg)\\
 & \leq \Oh( \rho^{-H(\tilde{D}+2\alpha)}\beta_n^2 )
\end{align*}
Now, if we find the largest $H$ (denoted by $\bar{H}$) such that the upper bound on $\sum_{h=0}^H|Q_h|$ given above is smaller than $n$, then $\bar{H}$ will be a lower bound on the maximum depth explored by the algorithm. 
From the definition of $\bar{H}$, we can show that there exists some constant $C'>0$, such that 
\[
 \bigg( \frac{C'\log(n)}{n} \bigg)^{1/(\tilde{D} + 2\alpha)} \leq \rho^{\bar{H}} \leq \frac{1}{\rho} \bigg( \frac{C'\log(n)}{n}\bigg)^{1/(\tilde{D} + 2\alpha)}
\]
Assuming that $n$ is large enough so that $\bar{H} \geq h_0$ and that $h_{\max}\geq \bar{H}$ (which is true if $h_{\max} \geq \frac{\log(n)}{2\alpha \log(1/\rho)}$), we can now upper bound the simple regret  as follows:
\begin{align*}
 \mathcal{S}_n &\leq (2N+1)V_{\bar{H}} \leq \Oh( \rho^{\bar{H}\alpha}\sqrt{\bar{H}} )\\
 & \leq \Oh \big( n^{-\frac{\alpha}{\tilde{D}+2\alpha}}(\log n)^{\frac{\alpha}{\tilde{D}+2\alpha}+\frac{1}{2}} \big)\\
 & \leq \tilde{\Oh}\big(n^{-\frac{\alpha}{\tilde{D}+2\alpha}}\big)
\end{align*}

\section{ Deferred Proofs from Section~\ref{subsec:bayesian_zooming_algo}}
\label{appendix:bayesian_zooming_algo}

\subsection{Details of the algorithm}

To complete the description of the algorithm, we need to calculate the terms $\beta_n$ and the term $(W(r_k))_{k\geq 0}$ for radius $r_k = diam(\X)2^{-k}$. 

We begin with the following simple claim which gives us the appropriate choice of $\beta_n$. 

\begin{claim}
 \label{claim:B_n_zoom}
 For the choice of $\beta_n  = \Oh(\sqrt{(D_1/\alpha+1)\log(n) + u})$, we have for any $u>0$:
 \[
  Pr( \Omega_{u3} ) \geq 1- e^{-u}
 \]
where the event $\Omega_{u3}$ is defined as
\begin{equation}
\label{eq:event_Omega_u3}
 \Omega_{u3} = \{ \forall t \leq t_n, \forall x \in A_t: |f(x)-\mu_{t-1}(x)| \leq \beta_n\sigma_{t-1}(x)  \}
\end{equation}
with  $t_n$ is the (random) number of rounds of the algorithm required for $n$ function evaluations.

\end{claim}

\begin{proof}

Let $M_n = M(\X,n^{-1/(2\alpha)},l)$ be the $n^{-1/(2\alpha)}$ packing number of $\X$ with respect to the metric $l$. Then by the design of the algorithm, at any time $t$ we have $|A_t| \leq M_n$, and also $t_n \leq M_n+n$ almost surely. 
So,  we get by two union bounds:
	\begin{align*}
	Pr(\Omega_{u3}^c) &\leq \sum_{t=1}^{t_n}\sum_{x \in A_t} 2e^{-\beta_n^2/2} \leq 2n^22e^{-\beta_n^2/2} \\
    & \leq 2M_n(M_n+n)e^{-\beta_n^2/2}
	\end{align*}
    
    Now, by using the fact that $\X$ has a finite metric dimension $D_1$ we have $M_n \leq Cn^{D_1/(2\alpha)}$ for some constant $C>0$. This implies that $M_n(M_n+n) \leq C^2n^{1+D_1/\alpha}$ for $n\geq 1$.

	Thus for any $u > 0$, the choice of $\beta_n = \sqrt{2(u+ 2\log(C) + (D_1/\alpha+1)\log(n)  )}$ ensures that $Pr(\Omega_{u3}) \geq 1-e^{-u}$. 
\end{proof}

Now, we obtain the terms $W(r_k)$ which denotes a high probability upper bound on the maximum variation in the GP sample within any ball of radius $r_k$ in $\X$. 

\begin{claim}
 \label{claim:zoom_r_k_W_k}
 Consider the choice of radius values $r_k = 2^{-k}diam(\X)$ for $k\geq 0$. Then  we have for any $u >0$,  
 $Pr(\Omega_{u4}) \geq 1- e^{-u}$, where the event $\Omega_{u4}$ is defined as:
	\begin{equation}
	 \label{eq:event_Omega_u4}
	 \Omega_{u4}= \{ \forall k\geq 0, \forall x \in \X: \sup_{y \in B(x,r_k,l)}|f(x)-f(y)| \leq W(r_k) \}	
	\end{equation}

The term  $W(r_k)$ is given by :
 \[
  W(r_k) = 2w(2r_k)   \leq 8g(r_k)\big( \sqrt{C_4 + 2u + 2\log(|\X_k|) + 2D_1\log(2^k/diam(\X))} + C_3 )
 \] with $\X_k$ being the $r_k$ cover of $\X$.

\end{claim}
\begin{proof}
 The result follows immediately by applying Proposition~\ref{prop:chaining3} with $\epsilon_k = r_k$ and
 $R_k = 2r_k$.
\end{proof}

Without loss of generality, we can assume that the diameter of the search space $\X$ is 1. 
Then, in the expression for $W(r_k)$ above, we can upper bound the term $|\X_k|$ for all $k$ by $C2^{kD_1}$ due to the assumption of finite metric dimension of $\X$. Thus for all $k \geq 0$ we have $W(r_k) \leq \Oh(g(r_k)\sqrt{u + 2D_1k\log(C^2/diam(\X))})$ and in particular for $k \geq \log_2(1/\delta_0)$ we have: 
\[
W(r_k) \leq \mathcal{O}( 2^{-k\alpha}\sqrt{u + 2D_1k\log(C^2/diam(\X))})
\]

Having described the algorithm parameters, we now present an outline of the derivation of the regret bounds for the Bayesian Zooming algorithm. We characterize the properties of the points selected by the algorithm in the 
following lemma. The proof of the regret bounds can be completed in an analogous manner to the proof of Theorem~\ref{theorem:regret_gp_bandits}

\begin{lemma}
\label{lemma:bayesian_zooming_algo}
Under the events $\Omega_{u3}$ and $\Omega_{u4}$, the following statements are true: 
\begin{itemize}
 \item Any point $x$ at which the function is evaluated by the algorithm satisfies: 
 \begin{equation}
  f(x^*) - f(x) \leq 5W(r(x)) 
 \end{equation}
\item If in round $t$, the function value is evaluated at a point $x$ with $r(x) > r_{min}$, then we have
\begin{equation}
 f(x^*) -f(x) \leq 3\beta_n\sigma_{t-1}(x)
\end{equation}

\item Any two points $x_1$ and $x_2$ which have been evaluated at least $k$ times each must satisfy $l(x_1,x_2) > r_k$. 

\item A point $x$ with radius $r_k$ will be evaluated no more than $q_{r_k}$ times before its radius is shrunk, where $q_{r_k}$ is defined as:
\begin{equation}
 q_{r_k} = \frac{\sigma^2\beta_n^2}{2W(r_k)^2}
\end{equation}

\end{itemize}

\end{lemma}

\begin{proof}
The results stated above follow directly from the point selection and refinement strategy used in the algorithm:
\begin{itemize} 
 \item Suppose $x_t^*$ denotes the point which contains the maximizer $x^*$ in its confidence region. Then we have the following: 
 \begin{align*}
 f(x^*) &\leq J_t(x_t^*) \leq J_t(x) \\
 & \leq f(x) + 2\beta_n \sigma_{t-1}(x) + W(r(x)) \\
 & \stackrel{(a)}{\leq} f(x) + 2W(2r(x)) + W(r(x)) \\
 & \leq f(x) + 5W(r(x))
 \end{align*}
where $(a)$ follows from the condition required for shrinking the radius associated with $x$ from $2r(x)$ to $r(x)$, assuming $r(x)<diam(\X)$. If $r(x) = diam(\X)$ then the inequality is trivially true by the definition of 
$W(r(x))$. 

\item If $r(x)$ is strictly greater than $r_{min}$ and the point $x$ is evaluated by the algorithm, then we must have $\beta_n\sigma_{t-1}(x) \geq W(r(x))$ which gives us the required inequality. 

\item Since the covering oracle only adds points from the uncovered region, the distance between two points with associated radius $r_k$ must be greater than $r_k$. 

\item Finally, the maximum number of times a point is evaluated by the algorithm before shrinking the radius is upper bounded by using the result in the first part of Proposition~\ref{prop:posterior_variance} to get the 
required expression of $q_{r_k}$. 
 \end{itemize}

\end{proof}
Having obtained the above results, we can retrace the steps in the proof of Theorem~\ref{theorem:regret_gp_bandits} to obtain similar regret bounds for the Bayesian Zooming algorithm.

\section*{Acknowledgements}
Tara Javidi would like to thank Galen Reeves for introducing her to the problem studied in this paper. Shubhanshu Shekhar would like to thank Emile Contal for helpful discussions regarding \citep{contal2016thesis}. The authors would also like to thank Jonathan Scarlett for helpful comments on an earlier version of the manuscript.

\bibliographystyle{apalike}
\bibliography{new_citations_bb}

\end{document}